\documentclass{article}

\usepackage[final]{nips_2017}
\usepackage{nips_2017}
\usepackage{amsmath}
\usepackage{amssymb}
\usepackage{amsthm}
\pdfoutput=1

\usepackage[utf8]{inputenc} 
\usepackage[T1]{fontenc}    
\usepackage{hyperref}       
\usepackage{url}            
\usepackage{booktabs}       
\usepackage{amsfonts}       
\usepackage{nicefrac}       
\usepackage{microtype}      
\usepackage{xcolor} 	
\usepackage{upgreek}  
\usepackage{thmtools}
\usepackage{thm-restate} 
\usepackage{graphicx}
\usepackage{float}

\usepackage{caption} 
\usepackage[export]{adjustbox} 

\title{Generalization in Deep Networks: \\ The Role of Distance from Initialization}

\author{
  Vaishnavh Nagarajan \\
  Computer Science Department\\
  Carnegie-Mellon University\\
  Pittsburgh, PA 15213 \\
  \texttt{vaishnavh@cs.cmu.edu}\\
\And
  J. Zico Kolter\\
  Computer Science Department\\
  Carnegie-Mellon University\\
  Pittsburgh, PA 15213 \\
  \texttt{zkolter@cs.cmu.edu}
}

\renewcommand{\vec}[1]{\boldsymbol{\mathbf{#1}}}

\newtheorem{theorem}{Theorem}[section]
\newtheorem{corollary}{Corollary}[theorem]

\newtheorem{definition}{Definition}[section]

\begin{document}

\maketitle

\begin{abstract}
Why does training deep neural networks using stochastic gradient descent (SGD)  result in a generalization error that does not worsen with the number of parameters in the network? 
To answer this question, we advocate 
a notion of effective model capacity that is dependent  on {\em a given random initialization of the network} and not just the training algorithm and the data distribution. We provide empirical evidences that demonstrate that the model capacity of SGD-trained deep networks is in fact restricted through implicit regularization of {\em the $\ell_2$ distance from the initialization}.  We also provide  theoretical arguments that further highlight the need for initialization-dependent notions of model capacity. We leave as open questions how and why distance from initialization is regularized, and whether it is sufficient to explain generalization.
\end{abstract}

\section{Introduction}
\label{sec:intro}

 \cite{neyshabur15iclr} and \cite{zhang17} proposed an important open question in deep learning that has been actively studied over the last year. The question stems from the observation that
an overparametrized deep model that can be trained to fit completely random labels, can also be trained to fit real-world data while achieving low test error. This contradicts the conventional understanding in machine learning about the tradeoff between the expressivity of a model and its ability to generalize.

In order to understand this phenomenon, it is now well-known that it is necessary to examine the capacity of the model in conjunction with the training algorithm, namely stochastic gradient descent (SGD). This can be studied from different perspectives, such as stability of the training algorithm \citep{hardt16} or through robustness of the solution found \citep{xu12,keskar17,dinh17}. 
Another direction is to identify some kind of {\em implicit regularization} that is performed by SGD \citep{neyshabur17}: are the parameters recovered by SGD restricted in a manner that can explain generalization? 
We study this question 
 specifically highlighting the need to examine the model capacity after fixing the random initialization of the network (Section~\ref{sec:problem}) as it was first done in the PAC-Bayesian bounds of \cite{dziugaite17}. We first provide empirical arguments showing that the distance of the learned network from its initialization -- a term that was  considered in \cite{dziugaite17} -- is implicitly regularized by SGD to a {width-independent} value (Section~\ref{sec:experiments}). Then, in Sections~\ref{sec:linear-network} and ~\ref{sec:bad-norms}, we provide theoretical arguments highlighting the role of an initialization-dependent model capacity -- such as the distance from initialization -- in explaining generalization.

\section{Notations}
\label{sec:notations}
We consider a feed-forward network of $d$ layers and $H$ hidden units in each layer, that maps from $\mathbb{R}^{n}$.   We denote the parameters by $\mathcal{W} = \{   \vec{W}_1, \vec{W}_2, \hdots, \vec{W}_d \}$ and $\mathcal{B} = \{ \vec{b}_1, \vec{b}_2, \hdots, \vec{b}_d \}$ so that the function computed by the network can be denoted as $f_{(\mathcal{W},\mathcal{B})}(\vec{x}) =  \vec{W}_d \phi \left( \vec{W}_{d-1} \phi(\hdots \phi(\vec{W}_1\vec{x} + \vec{b}_1) \hdots ) +  \vec{b}_{d-1} \right) +\vec{b}_d $ for all $\vec{x} \in \mathbb{R}^n$.  Note that we will consider only a single output network for the sake of our discussion, and thus 
 $\vec{W}_d$ is only a column matrix and $\vec{b}_d$ is a scalar value.  We will denote the output of the $k$th layer before applying the activation function as $f^{(k)}_{(\mathcal{W},\mathcal{B})}(\vec{x}) = \vec{W}_k \phi(f^{(k-1)}_{(\mathcal{W},\mathcal{B})}(\vec{x})) + \vec{b}_k$, with $f^{(0)}_{(\mathcal{W},\mathcal{B})}(\vec{x})=\vec{x}$.

Here, $\phi$ is the activation function, which we will assume to be ReLU. Furthermore, we will assume that the weights are initialized as $(\mathcal{Z}, \mathcal{C})$ where $\mathcal{Z}$ is initialized according to Xavier initialization and $\mathcal{C} = 0$.
We will focus on networks of depth greater than $2$, which would mean that each parameter in $\mathcal{Z}$ is drawn independently from $\mathcal{N}(0,{\Theta}(1/\sqrt{H}))$. 

We will use $\| \cdot \|_F$ to denote the Frobenius norm
and  $\| \cdot \|_2$  the spectral norm. For a vector, we will use $\| \cdot \|$ to denote its $\ell_2$ norm. Furthermore, we will use the notation $\|(\mathcal{W}, \mathcal{B}) - (\mathcal{Z}, \mathcal{C}) \|_F $ to denote $\sqrt{\sum_{k=1}^{d} \|\vec{W}_k - \vec{Z}_k \|_{F}^2 +  \|\vec{b}_k - \vec{c}_k \|^2 }$.  We use a tilde symbol over asymptotic notations like, $\tilde{O}$ and $\tilde{\Omega}$, to ignore logarithmic factors in the bounds.

\section{Problem Formulation}
\label{sec:problem}

We first formally define a notion of effective model capacity based on which we will study implicit regularization and generalization. Let $\mathcal{D}$ be a distribution over data in $\mathbb{R}^{n} \times [0,1]$. Let $\mathcal{L}: \mathbb{R}^n \times \mathbb{R} \to [0,1]$ be a loss function.  Given $m$ datapoints, our goal is to explain why training a network of $H$ hidden units per layer using SGD to zero training loss, 
 results in a generalization error dependent only on $m$ and independent of $H$.  In order to do this, we argue (through empirical and theoretical evidences) that the effective capacity of the model must be take into account not just the training algorithm and the underlying distribution but also {\em the given random initialization}:

\begin{definition}  \label{def:effective-capacity} For a particular model (i.e., network architecture), we define the   {\em \bf effective capacity}  $\mathcal{H}_{m,\delta}[\mathcal{D}, (\mathcal{Z}, \mathcal{C}), \mathcal{A}]$ of a (distribution, initialization, algorithm)-tuple to be a set of parameter configurations such that with high probability $1-\delta$ over the draws of a dataset of $m$ samples from the distribution $\mathcal{D}$, training the network initialized with $ (\mathcal{Z}, \mathcal{C})$,  to zero loss on these samples using algorithm \footnote{If $\mathcal{A}$ is stochastic, we could either incorporate it in the ``high probability'' component of the definition, or we could ``freeze'' it by including it as an argument to $\mathcal{H}$ like the random initialization.} $\mathcal{A}$, obtains a parameter configuration that lies in the set  $\mathcal{H}_{m,\delta}[\mathcal{D}, (\mathcal{Z}, \mathcal{C}), \mathcal{A}]$. 
\end{definition}

Note that this notion of effective capacity reflects the idea of \cite{dziugaite17} who incorporated the initialization into their analysis by arguing that it will take into account the symmetries of the network. On the other hand, this notion is
more refined than the one in \cite{arpit17}  which is independent of both the data distribution and the (random) initializations. Similarly, \cite{neyshabur17} consider an indirect notion of effective capacity by evaluating a norm (such as the $\ell_2$ norm) for each parameter configuration and investigating whether the algorithm restricts itself to configurations of low norm; however, these norms are calculated independent of the random initialization.

Our goal now is to identify as precise a characterization of $\mathcal{H}_{m,\delta}[\mathcal{D}, (\mathcal{Z}, \mathcal{C}), \mathcal{A}]$ as possible, hopefully one that has a learning-theoretic complexity independent of the network width $H$. This would then allow one to derive $H$-independent generalization guarantees
 (see Lemma~\ref{lem:rademacher}).
Effectively, instead of seeking a `global' quantity regularized across all initializations (such as the $\ell_2$ norm in \cite{neyshabur17}), we seek one that is specific to the initialization.

The focus of our paper is arguably the simplest such quantity, one that was originally considered  in \cite{dziugaite17}: the distance of the weights $(\mathcal{W}, \mathcal{B})$ from the initialization, $\| (\mathcal{W}, \mathcal{B}) - (\mathcal{Z}, \mathcal{C}) \|_F$. Specifically, \cite{dziugaite17} presented a PAC-Bayesian bound involving the distance from initialization (rather than one involving the distance from origin) and showed that SGD can be made to explicitly regularize such a bound in a way that a non-vacuous PAC-Bayesian bound holds on the resulting network. We show that distance from initialization is in fact {\em implicitly} regularized by SGD and we investigate it in much greater detail in terms of its dependence on the parameter count and its ability to explain generalization.

\section{Experiments}
\label{sec:experiments}

We conduct experiments on the CIFAR-10 \& MNIST datasets, where we train networks of 4 hidden layers with varying width $H$ to minimize cross entropy loss and squared error loss. We study how distance from initialization (which we will denote in short as $r$) varies with width $H$ and training set size $m$ both for real data and partially/fully corrupted labels like in \cite{zhang17} and \cite{arpit17}.   We summarize our observations below and present more experiments and details in Appendix~\ref{app:experiments}.

 First, for both the real data and the corrupted data, the distance $r$ mostly remains constant {\em or even decreases} with width $H$ when $H$ is not too large (sometimes for very large $H$, $r$ shows only a slight increase that scales logarithmically with $H$).  For reference, see Figure~\ref{fig:cross-entropy} (left)  and sub-figure (c) in Figures ~\ref{fig:mnist-momentum},
~\ref{fig:cifar}, and ~\ref{fig:mnist}.

Second,  as we can see in Figures~\ref{fig:cross-entropy} (right), $r$ increases with more noise in the labels, and this increase is more pronounced when sample size $m$ is larger. This demonstrates that 
 larger distances need to be traveled  in order to achieve stronger levels of memorization. 
For further reference, see sub-figures (a) and (b) in  Figure~\ref{fig:noise} and the figures in Figure~\ref{fig:noise} and Figure~\ref{fig:varied-noise} for additional reference). 
  
  Finally, we note that even though $r$ is regularized to a width-independent value, it does
{\em grow with the training set size $m$}, typically at between the rates of  $m^{0.25}$ to $m^{0.4}$ (when there is no noise in the training data). For reference, see Figure~\ref{fig:cross-entropy} (left)  and sub-figure (d) in Figures ~\ref{fig:mnist-momentum},
~\ref{fig:cifar}, and ~\ref{fig:mnist}. The growth rate is more prominent for smaller $H$ or when there is more noise in the labels (see sub-figures (b) and (d) in Figure~\ref{fig:noise} and the figures in Figure~\ref{fig:noise} for additional reference).

\begin{figure}
        \begin{minipage}{.5\textwidth}
        \centering
        \includegraphics[scale=0.25]{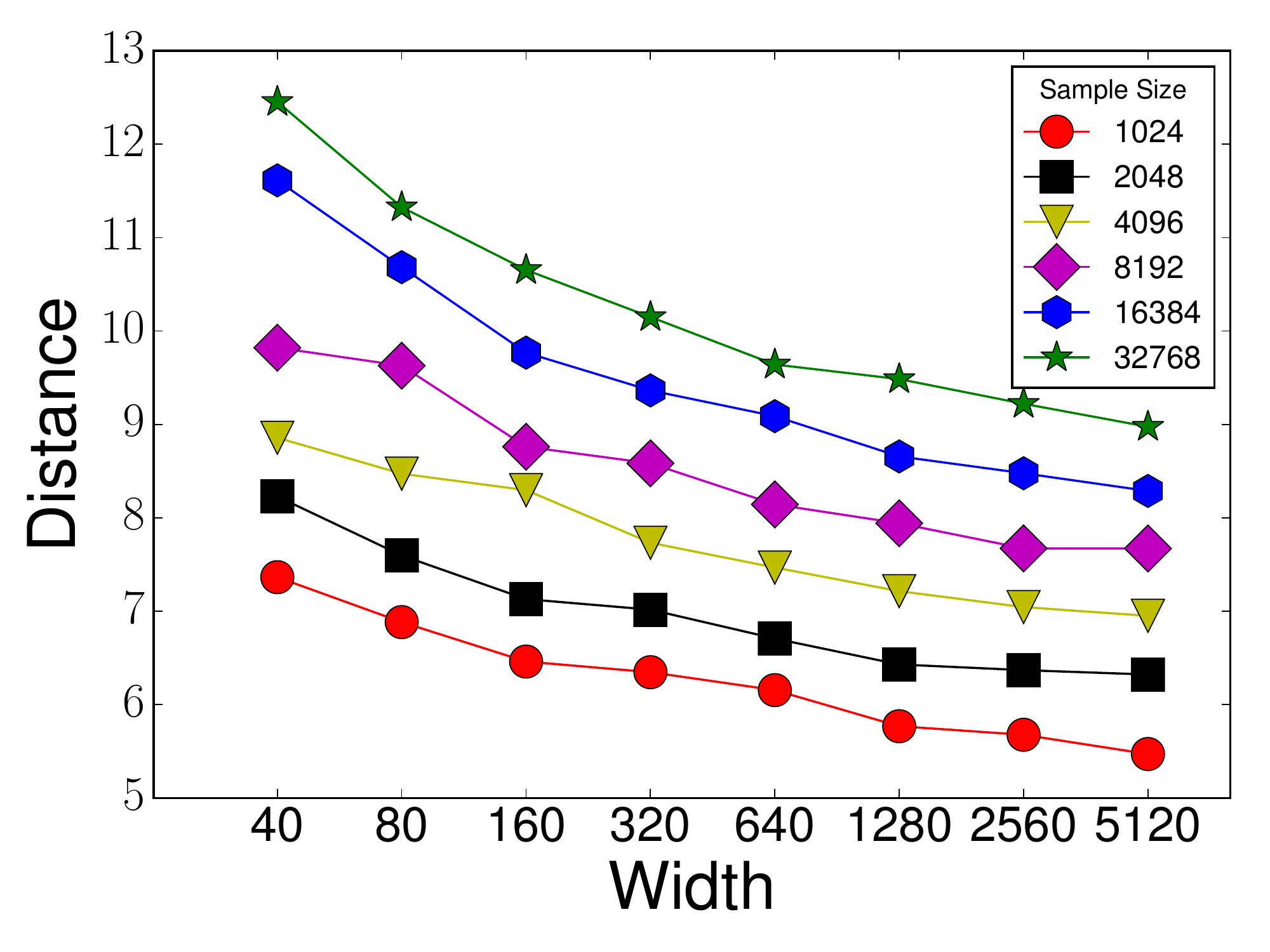}\\
    \end{minipage}%
        \begin{minipage}{.5\textwidth}
        \centering
        \includegraphics[scale=0.25]{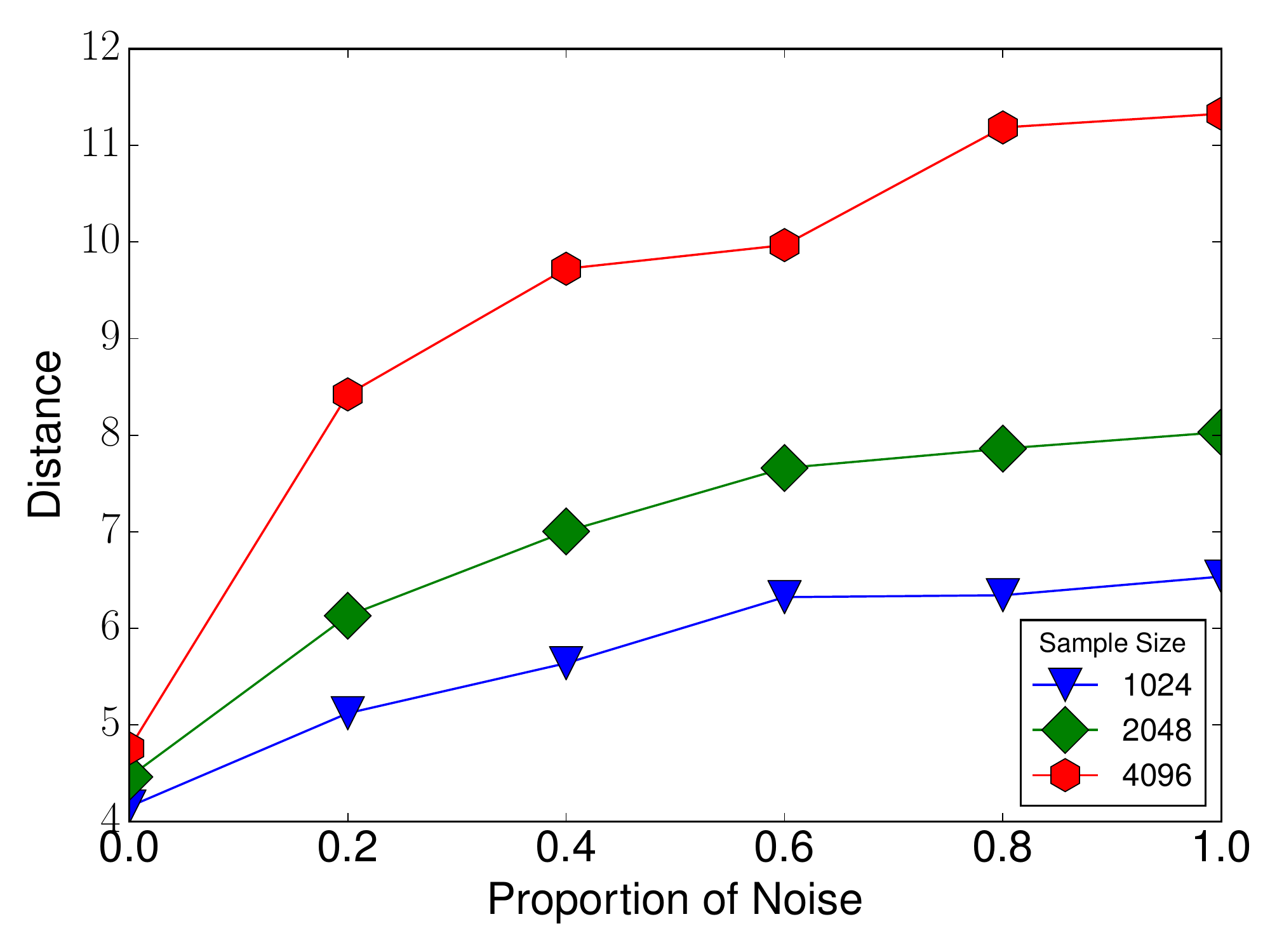}\\
    \end{minipage}\hfill
	\caption{Experiments on a 5 layer neural network trained to minimize the cross-entropy loss on the MNIST dataset until at least $99\%$ of the data is classified by a margin of at least $10$ i.e., the output on the correct class is at least $10$ larger than all the other classes. 
	\textbf{Left}: Distance from initialization is regularized to a $H$-independent value, and in fact, decreases with it. \textbf{Right}: Distance from initialization grows with the proportion of noise in the labels.}
    \label{fig:cross-entropy}
\end{figure}

\section{Complexity of linear networks within fixed distance from initialization}
\label{sec:linear-network}
Is the distance regularization observed above sufficient to explain generalization?  While many norm-based generalization bounds have already been derived for ReLU networks \citep{neyshabur15,neyshabur18pacbayes} which can be potentially improved with this observation, it seems non-trivial to prove an $H$-independent generalization bound with this observation alone. For example, it is easy to incorporate this quantity (in place of distance from the origin) in PAC-Bayesian analysis such as \cite{neyshabur18pacbayes}, as was already done in \cite{dziugaite17}. While this would result in bounds that are tighter by a factor of $\sqrt{H}$ (because distance from origin grows as $\sqrt{H}$, as we will discuss shortly in Section~\ref{sec:bad-norms}), the resulting bound still has dependence on the network width.

As a first step to test the usefulness of the observed distance regularization, we consider a network with all hidden units as simple linear units. If the Rademacher complexity of this space of networks was {\em not} independent of width $H$, then there would be no hope in expecting the same space but with non-linearities to have a width-independent complexity. Fortunately, we can show that this is not the case for linear networks. Our proof is based on how a Xavier-initialized network has weight matrices with width-independent spectral norms with high probability (see Corollary~\ref{cor:bounds}). We defer the proof for this to Appendix~\ref{app:proofs} (and present some additional results in Appendix~\ref{app:cool-properties}).

\begin{restatable}{theorem}{linearnetwork}
\label{thm:linear-network}
 For the network described in Section~\ref{sec:notations} but with linear activations,
 the Rademacher complexity corresponding to the hypotheses that are at an $\ell_2$ distance of at most $r$ satisfies:
\[
\mathbb{E}_{\vec{\xi}} \left[ \sup_{(\mathcal{W}, \mathcal{B}): \| (\mathcal{Z}, \mathcal{C}) -  (\mathcal{W}, \mathcal{B})\|_F \leq r } \sum_{i=1}^{m}  \xi_i f_{(\mathcal{W}, \mathcal{B})}(\vec{x}_i) \right] = \tilde{\mathcal{O}} \left( \frac{d c^d (r+1)^{d}\max_i \| \vec{x}_i\|}{\sqrt{m}} \right),
\]
where $c=\tilde{\Theta}(1)$.  

\end{restatable}


\section{Why initialization-independent norms may not explain generalization}
\label{sec:bad-norms}

We now go back and look at some norms studied in \cite{neyshabur17} and evaluate why they were unable to explain generalization.
First, consider the product of $\ell_2$ norms proportional
to $\prod_{i=1}^{d}\| \vec{W}_k\|_{F}^2$,
which they observe increases with the width $H$ for large $H$. Unfortunately, this could not explain why generalization error is width-independent because the best known bound on the Rademacher complexity of the class of $\ell_2$-norm-bounded networks grows with the norm bound and hence also grows with $H$ \citep{neyshabur15}. 
Through the proposition below, we present a more theoretically grounded perspective as to why this norm may not explain generalization:

\begin{restatable}{proposition}{badnorms}
With high probability over the draws of the random initialization, even though the untrained network provably has a $H$-independent generalization error $\tilde{\mathcal{O}}\left( 1/\sqrt{m}\right)$, its $\ell_2$ norm $\prod_{k=1}^{d}\| \vec{Z}_k \|_{F}^2$ grows as $\tilde{\Omega}(H^{d-2})$. 
\end{restatable}

The main takeaway from the above proposition  is that a norm-based capacity measure for neural networks may not be useful for explaining generalization if it is blind to the random initialization and instead measures any kind of distance of the weights from the origin. This is because, {\em for larger $H$, most random initializations and the origin, all lie farther and farther away from each other}. Therefore, it may not be reasonable to expect that for these initializations, SGD goes all the way close to the origin to find solutions.

\textbf{On the spectral norm}:  \cite{neyshabur17} also study a spectral norm proportional to $H^{d-1} \prod_{k=1}^{d} \| \vec{W}_k \|_2$. First we note that, like the $\ell_2$ norm above, even this grows with $H$ as $\tilde{\Omega}(H^{d-1})$ for the random initialization, because the spectral norm of the random matrices are $\tilde{\Theta}(1)$.  \cite{neyshabur17} then ask whether the factor of $H^{d-1}$ is ``necessary'' in this measure or not (in Section 2.2 of their paper). Formally, we frame this question as: is it sufficient if $\prod_{k=1}^{d} \| \vec{W}_k\|_2$ is regularized to an $H$-independent value (as against ensuring the same for $H^{d-1} \prod_{k=1}^{d} \| \vec{W}_k\|_2$) for guaranteeing $H$-independent generalization?  

Through our proposition below,  we argue that a better question can be asked. In particular, we  show that if distance from initialization is regularized to a $H$-independent value, then so is the term $\prod_{k=1}^{d} \| \vec{W}_k \|_2$. Therefore, $H$-independent spectral norms are less useful for the purpose of explaining generalization; we would rather want to answer whether it is sufficient if the distance norm is regularized to an $H$-independent value to guarantee $H$-independent generalization. Effectively, this would 
boil down to extending Theorem~\ref{thm:linear-network} to non-linear networks.

\begin{restatable}{proposition}{spectral}
$\prod_{k=1}^{d} \| \vec{W}_k \|_2 \leq \tilde{\mathcal{O}}\left(c^d\left(1+\|(\mathcal{W}, \mathcal{B}) - (\mathcal{Z},\mathcal{C})\|_F\right)^{d} \right)$ for some $c = \tilde{\mathcal{O}}(1)$.
\end{restatable}

\section{Conclusion}
\label{sec:conclusions}
To explain generalization in deep networks, we highlight the need to understand the effective capacity of a model for a given random initialization of the network. Furthermore, our experiments suggest that distance moved by the training algorithm from its random initialization has a key role to play in this regard. This leads to multiple concrete open questions for future work.
First, why is distance from the initialization regularized by the training algorithm? Can we precisely bound this distance independent of the number of hidden units, $H$? 
Next, is this observation alone sufficient to explain generalization? More concretely, can we prove an $H$-independent bound on the empirical Rademacher complexity (or any other learning-theoretic complexity) for distance-regularized networks, like we could for linear networks in Theorem~\ref{thm:linear-network}?
If that is not possible, can we identify a more precise characterization of the effective capacity as defined in Definition~\ref{def:effective-capacity}? That is, {\em for a fixed random initialization, do the solutions obtained by the training algorithm on most training sets lie within a smaller subspace inside a ball of $H$-independent radius around the random initialization? }

\bibliographystyle{plainnat}
\bibliography{references}

\appendix

\section{Useful Lemmas}
\label{app:useful-lemmas}
In this section, we present some standard facts and some corollaries relevant for our discussion. First, we make use of the following theorem on squares of normal random variables to derive a bound on the Frobenius norm of a matrix with random gaussian entries.

\begin{theorem}
\label{thm:normsq}
If $z_{1}, z_2, \hdots, z_k \sim \mathcal{N}(0,1)$ be $k$ standard normal random variables drawn independently, then for all $t \in (0,1)$:
\[
Pr\left[ \left|\frac{1}{k}\sum_{i=1}^{k} z_{i}^2 - 1 \right| \geq t \right] \leq 2e^{-kt^2/8}  
\]
\end{theorem} 
%
%

Next, we use the following theorem based on Theorem 2.1.1 in \cite{wainwright}, to bound the spectral norm of a matrix with random gaussian entries.
\begin{theorem}
Let $\vec{W}$ be a $n_1 \times n_2$ matrix with entries all drawn independently at random from $\mathcal{N}(0,1)$. Then, given $n_1 \geq n_2$, 
\[
Pr \left[ \left| \sup_{u: \| u\| = 1}   \frac{\| \vec{W} u  \|^2}{n_1} - 1 \right| \geq \mathcal{O}\left(\sqrt{\frac{n_2}{n_1} } + t\right) \right] \leq \mathcal{O} \left( e^{-n_1t^2/2}\right)
\]
\end{theorem}

As a result of the above theorems, we have that:
\begin{corollary}
\label{cor:bounds}
For a network of more than $1$ hidden layer, and $H$ hidden units per layer, when its initialization $\mathcal{Z}$ is according to Xavier i.e.,  when all entries are drawn from $\mathcal{N}(0,\mathcal{O}(1/\sqrt{H}))$, with high probability we have: 
\begin{itemize}
\item $\|\vec{Z}_1\|_F = \tilde{{\Theta}}(\sqrt{n})$, $\|\vec{Z}_k\|_F = \tilde{{\Theta}}(\sqrt{H})$ for $d > k > 1$, $\|\vec{Z}_d\|_F = \tilde{{\Theta}}(1)$.
\item  $\|\vec{Z}_k\|_2 = \tilde{{\Theta}}(1)$ for $k$.
\end{itemize}
\end{corollary}

Next, we present the Khintchine-Kahane inequality which we use to bound Rademacher complexity in the absence of the supremum within the expectation.

\begin{theorem}
\label{thm:kk}
For any $0 < p < \infty$ and set of scalar values $ \{ x_1, x_2, \hdots, x_m\}$, when $\vec{\xi}$ is a Rademacher vector sampled uniformly from $\{ -1, 1\}^m$:
\[
\left( \mathbb{E}_{\vec{\xi}} \left[\sum_{i=1}^{m} \left| \xi_i x_i \right|^p\right]\right)^{1/p} \leq C_{p} \left( \sum_{i=1}^{m} {x}_i^2\right)^{1/2}
\]
where $C_p$ is a constant dependent only on $p$.
\end{theorem}

It is simple to extend this to vector-valued variables for $p=2$, which is what we will need specifically for our discussion:

\begin{corollary}
\label{cor:kk}
For a set of vectors $ \{ \vec{x}_1, \vec{x}_2, \hdots, \vec{x}_m\}$, when $\vec{\xi}$ is a Rademacher vector sampled uniformly from $\{ -1, 1\}^m$:
\[
\mathbb{E}_{\vec{\xi}} \left[\left\| \sum_{i=1}^{m}  \xi_i \vec{x}_i \right\|\right] \leq \left( \mathbb{E}_{\vec{\xi}} \left[ \left\|\sum_{i=1}^{m}  \xi_i\vec{x}_i \right\|^2\right]\right)^{1/2} \leq C_{p} \left( \sum_{i=1}^{m} \|\vec{x}_i\|^2\right)^{1/2}
\]
\end{corollary}

Here, the first inequality follows from Jensen's inequality. To derive the next, we first apply the Khintchine-Kahane inequality for each dimension of these vectors. We then square the inequalities and sum them up, after which we take the square root of both sides.

\section{Proofs}
\label{app:proofs}
We now present proofs for the results stated in the main paper. First we begin with the simple lemma, referred to in Section~\ref{sec:problem}, that explains why it is sufficient to study an effective capacity of the model that is dependent on the random initialization, the algorithm and the underlying data distribution, in order to understand generalization.

\begin{restatable}{lemma}{rademacher}
\label{lem:rademacher}
If with probability $1- \frac{\delta}{2}$ over the random draws of the initialization $(\mathcal{Z}, \mathcal{C})$ and the random draws of $m$ samples $\mathcal{S} = \{ (\vec{x}_i, y_i) | i=1,\hdots m \}$, the empirical Rademacher complexity of the space of parameters $\mathcal{H}_{m,\delta/2}[\mathcal{D}, (\mathcal{Z}, \mathcal{C}), \mathcal{A}]$ with respect to the loss function over the samples satisfies  the following $H$-independent bound:
\[
\frac{1}{m}\mathbb{E}_{\vec{\xi}} \left[ \sup_{(\mathcal{W}, \mathcal{B}) \in \mathcal{H}_{m,\delta/2}[\mathcal{D}, (\mathcal{Z}, \mathcal{C}), \mathcal{A}]} \sum_{i=1}^{m}  \xi_i \mathcal{L}(f_{(\mathcal{W}, \mathcal{B})}(\vec{x}_i), y_i) \right] = \mathcal{O} \left( poly\left(\frac{1}{m} , \log \frac{1}{\delta}\right)\right),
\]
then with probability $1-\delta$ over the random initialization and random draws of $m$ samples $\mathcal{S}$, the generalization error of the hypothesis $(\mathcal{W}^\star, \mathcal{B}^\star)$ output by algorithm $\mathcal{A}$ trained on $\mathcal{S}$ satisfies:
\[
\mathbb{E}_{(\vec{x},y) \sim \mathcal{D}} \left[ \mathcal{L}(f_{(\mathcal{W}^\star, \mathcal{B}^\star)}(\vec{x}), y)\right]- \frac{1}{m}\sum_{i=1}^{m} \mathcal{L}(f_{(\mathcal{W}^\star, \mathcal{B}^\star)}(\vec{x}_i), y_i) = \mathcal{O} \left( poly\left(\frac{1}{m} , \log \frac{1}{\delta}\right)\right)
\] 

\end{restatable}


\begin{proof}
We will use $\mathcal{H}$ to denote $\mathcal{H}_{m,\delta}[\mathcal{D}, (\mathcal{Z}, \mathcal{C}), \mathcal{A}]$ in short. Furthermore, we will use $\mathcal{H}$ to refer to the set of parameter configurations or the set of functions/hypotheses corresponding to those configurations interchangeably.

Consider a random draw of the initialization and $m$ samples $\mathcal{S}$. The precondition in the above lemma is essentially a  bound on the complexity of $\mathcal{H}$ that is independent of $H$. When this bound holds, then from standard results in learning theory, this implies ``uniform convergence'' of the loss function for every parameter in the space $\mathcal{H}$. That is, for {\em every} parameter $(\mathcal{W}, \mathcal{B}) \in \mathcal{H}$, we have that the empirical estimate of its loss on $\mathcal{S}$ approximates the expected value with an $H$-independent error. Formally, with probability $1- \delta/2$, $\forall (\mathcal{W}, \mathcal{B}) \in \mathcal{H}$
\[
\mathbb{E}_{(\vec{x},y) \sim \mathcal{D}} \left[ \mathcal{L}(f_{(\mathcal{W}, \mathcal{B})}(\vec{x}), y)\right]- \frac{1}{m}\sum_{i=1}^{m} \mathcal{L}(f_{(\mathcal{W}, \mathcal{B})}(\vec{x}_i), y_i) = \mathcal{O} \left( poly\left(\frac{1}{m} , \log \frac{1}{\delta}\right)\right)
\] 

Next, by the definition of $\mathcal{H}$ we have that with probability $1-\delta/2$, $(\mathcal{W}^\star, \mathcal{B}^\star)$ falls in $\mathcal{H}$. Therefore, by instantiating the above expression for $(\mathcal{W}^\star, \mathcal{B}^\star)$, we have our generalization bound which by the union bound, holds with probability $1-\delta$.
\end{proof}

\subsection{Proofs from Section~\ref{sec:linear-network}}
Next, we present the proof for our result on $H$-independent bound on the complexity of linear networks that lie within fixed distance from a given random initialization. 

\linearnetwork*

\begin{proof}
Crucial to our proof is Corollary~\ref{cor:bounds}. Specifically, for the random initialization $\mathcal{Z}$, with high probability, we can bound the spectral norms of all the $H\times H$ matrices in $\mathcal{Z}$ as $\| \vec{Z}_k \|_2 = \tilde{\Theta}(1)$ (for $1 < k\leq d$) and $\| \vec{Z}_1\|_2 =  \tilde{\Theta}(\sqrt{n})$. We present these and a few other relevant bounds in Corollary~\ref{cor:bounds}.

Now our approach is to remove the network parameters in the expression for the Rademacher complexity layer by layer while applying this bound. For shorthand, we will simply write $\sup$ to denote the supremum over the space ${(\mathcal{W}, \mathcal{B}): \| (\mathcal{Z}, \mathcal{C}) -  (\mathcal{W}, \mathcal{B})\| \leq r }$. Then, we get the following recursive bound for $k > 1$:

\begin{align*}
 \mathbb{E}_{\vec{\xi}} \left[ \sup  \left \| \sum_{i=1}^{m}  \xi_i   f^{(k)}_{(\mathcal{W}, \mathcal{B})}(\vec{x}_i) \right\| \right]  & =  \mathbb{E}_{\vec{\xi}} \left[ \sup  \left \| \sum_{i=1}^{m}  \xi_i \left( \vec{W}_{k} f^{(k-1)}_{(\mathcal{W}, \mathcal{B})}(\vec{x}_i) + \vec{b}_{k} \right)\right\| \right] \\
 & \leq \mathbb{E}_{\vec{\xi}} \left[ \sup  \left \| \sum_{i=1}^{m}  \xi_i \vec{W}_{k} f^{(k-1)}_{(\mathcal{W}, \mathcal{B})}(\vec{x}_i)  \right\| \right] + \mathbb{E}_{\vec{\xi}} \left[ \sup  \left \| \sum_{i=1}^{m}  \xi_i  \vec{b}_{k} \right\| \right] \\
 & \leq \mathbb{E}_{\vec{\xi}} \left[ \sup \|\vec{W}_{k} \|_2  \left \| \sum_{i=1}^{m}  \xi_i  f^{(k-1)}_{(\mathcal{W}, \mathcal{B})}(\vec{x}_i)  \right\| \right] + \mathbb{E}_{\vec{\xi}} \left[ \sup  \left \| \sum_{i=1}^{m}  \xi_i  \vec{b}_{k} \right\| \right] \\
 & \leq \tilde{\mathcal{O}}(\|\vec{W}_k - \vec{Z}_k \|_F + \| \vec{Z}_k \|_2)   \mathbb{E}_{\vec{\xi}} \left[ \sup \left \| \sum_{i=1}^{m}  \xi_i  f^{(k-1)}_{(\mathcal{W}, \mathcal{B})}(\vec{x}_i)  \right\| \right] \\
 & +  \tilde{\mathcal{O}}(\|\vec{b}_k \|_2 )  \mathbb{E}_{\vec{\xi}} \left[ \sup  \left \| \sum_{i=1}^{m}  \xi_i  \right\| \right] \\
 & \leq   \tilde{\mathcal{O}}(r + \|\vec{Z}_k \|_2)   \mathbb{E}_{\vec{\xi}} \left[ \left \| \sum_{i=1}^{m}  \xi_i  f^{(k-1)}_{(\mathcal{W}, \mathcal{B})}(\vec{x}_i)  \right\| \right] +  \tilde{\mathcal{O}}(r ) \sqrt{m} \\
\end{align*}

Above, we have used the Khintchine-Kahane inequality (see Theorem~\ref{thm:kk} and Corollary~\ref{cor:kk})  to bound $\mathbb{E}_{\vec{\xi}} \left[ \sup  \left \| \sum_{i=1}^{m}  \xi_i  \right \|  \right] $. For $k=0$, we get:
\begin{align*}
\mathbb{E}_{\vec{\xi}} \left[ \sup  \left \| \sum_{i=1}^{m}  \xi_i   f^{(k)}_{(\mathcal{W}, \mathcal{B})}(\vec{x}_i) \right\| \right]   = \mathbb{E}_{\vec{\xi}} \left[ \sup  \left \| \sum_{i=1}^{m}  \xi_i   \vec{x}_i \right\| \right]  \leq \mathcal{O}(\sqrt{\sum \|\vec{x}_i \|^2}) \leq \tilde{\mathcal{O}}(\sqrt{m} \max_i \| \vec{x}\|_i)
\end{align*}

Here again we have used the Khintchine-Kahane inequality to bound $\mathbb{E}_{\vec{\xi}} \left[ \sup  \left \| \sum_{i=1}^{m}  \xi_i   \vec{x}_i \right\| \right] $. Finally, our claim then follows from repeated application of these recursive bounds. We have included the linear factor of $d$ to account for the term $ \tilde{\mathcal{O}}(r  \sqrt{m} )$ that is added due to the biases in each layer. Similarly the constant $c$ corresponds to the constant within the asymptotic bounds obtained in each recursive application of the above bound. 
\end{proof}

\subsection{Proofs from Section~\ref{sec:bad-norms}}

Now, we present our proof for why the generalization error of the random initialization is independent of $H$, while its $\ell_2$ norm grows with $H$.

\badnorms*

\begin{proof} 
In the terminology of Definition~\ref{def:effective-capacity},  for a random initialization, the effective capacity when the algorithm is simply one which outputs the initialization itself, is the singleton set consisting of that initialization i.e., $\mathcal{H}_{m,\delta}[\mathcal{D},(\mathcal{Z},\mathcal{C}),\mathcal{A}] = \{(\mathcal{Z}, \mathcal{C}) \}$. The generalization error of this algorithm then follows from a simple application of standard concentration inequalities of bounded random variables, with the random variables here being the loss of this network on a $m$ random i.i.d `training' input from the underlying distribution. For standard 0-1 error, this random variable is by default bounded. We can also show that the squared error loss is bounded to a width-independent value, since the output of this randomly initialized network is bounded independent of $H$ (Theorem~\ref{thm:output-bound}). Thus the generalization error of this network is width-independent.

The second part of our claim follows directly from the Frobenius norm bounds in Corollary~\ref{cor:bounds}.
\end{proof}

Next, we show why regularization of distance from initialization is more powerful than regularization of spectral norms, in terms of its ability to explain generalization independent of the number of hidden units per layer $H$. 

\spectral*

\begin{proof}

For any $k$, we have that $\| \vec{W}_{k} \|_2 \leq \| \vec{Z}_{k}\|_2 + \|\vec{Z}_k - \vec{W}_k \|_F \leq  \| \vec{Z}_{k}\|_2+ \|(\mathcal{W}, \mathcal{B}) - (\mathcal{Z},\mathcal{C})\|_F$.  The result then immediately follows from the spectral norm bounds in Corollary~\ref{cor:bounds}.

\end{proof}

\section{Some $H$-independent bounds for ReLU networks}
\label{app:cool-properties}
In this section we lay out two useful properties of neural networks in terms of how far away their weight are from their random initialization. In particular, we show that both the output and the gradient of a network with respect to its parameters is bounded purely by the distance from its random initialization and not on the number of hidden units. As always, we assume that the initialization $\mathcal{Z}$ is according to Xavier initialization (i.e., in this case the weights are drawn from a zero-mean gaussian with standard deviation $\mathcal{O}({1}/{\sqrt{H}})$) and $\mathcal{C}$ is zero. 

\begin{theorem}
\label{thm:output-bound}
$|f_{(\mathcal{W}, \mathcal{B})}(\vec{x})| \leq  \tilde{\mathcal{O}}(c^d(r+1)^{d}(\|\vec{x} \|+1)) $, where $r = \| (\mathcal{Z}, \mathcal{C}) - (\mathcal{W},\mathcal{B}) \|_F$ and $c=\tilde{\mathcal{O}}(1)$.
\end{theorem}

\begin{proof}
We will bound the magnitude of the output $f^{(k)}_{(\mathcal{W}, \mathcal{B})}(\vec{x})$ as follows: 

\begin{align*}
\| f^{(k)}_{(\mathcal{W}, \mathcal{B})}(\vec{x}) \| & =  \| \vec{W}_{k}\phi(f^{(k-1)}_{(\mathcal{W}, \mathcal{B})}(\vec{x}))  + \vec{b}_{k} \|  \leq \|\vec{W}_{k}\phi(f^{(k-1)}_{(\mathcal{W}, \mathcal{B})}(\vec{x})) \| + \|\vec{b}_k \|\\ 
& \leq \|\vec{W}_k \|_2 \| \phi(f^{(k-1)}_{(\mathcal{W}, \mathcal{B})}(\vec{x})) \|_2 + r \\
& \leq (\|\vec{W}_k  - \vec{Z}_k \|_F + \| \vec{Z}_k \|_2) \| f^{(k-1)}_{(\mathcal{W}, \mathcal{B})}(\vec{x}) \|_2 + r \\
& \leq \mathcal{O}(r + 1 ) \| f^{(k-1)}_{(\mathcal{W}, \mathcal{B})}(\vec{x}) \|_2 + r \\
\end{align*}

In the third line above, we use the fact that for any scalar value, $|\phi(x)| \leq |x|$ as $\phi$ is the ReLU activation. Following that, we use the spectral norm bounds from Corollary~\ref{cor:bounds}. 
Our bound then follows from repeated applications of these bounds recursively. Note that the value $c$ corresponds to the constant present in the asymptotic bound applied in each recursion.
\end{proof}

As a corrollary, we can bound the initial squared error loss of the network on a set of datapoints, independent of $H$:
\begin{corollary}
\label{cor:loss-bound}
Let ${(\vec{x}_1, {y}_1), \hdots, (\vec{x}_m, {y}_m)}$ be a set of training datapoints. For a randomly initialized network of any size, with high probability, the initial loss can be bounded independent of $H$ as
\[
\frac{1}{m}\sum_{i=1}^{m} (f_{(\mathcal{Z}, \mathcal{C})}(\vec{x}) - y)^2 \leq \left(\tilde{\mathcal{O}}(c^d (\max_{i}\|\vec{x}_i \|+1)) + \max_i |y_i| \right)^2
\]
\end{corollary}

Next, we bound the gradient of the function with respect to the parameters $\mathcal{W}$, independent of $H$. 
\begin{theorem}
\[
\left\| \frac{ \partial f_{(\mathcal{W}, \mathcal{B})}(\vec{x})}{\partial \mathcal{W}}\right\|   \leq  \tilde{\mathcal{O}}( d c^d(r+1)^{d}(\|\vec{x} \|+1)) 
\]
where $r = \| (\mathcal{Z}, \mathcal{C}) - (\mathcal{W},\mathcal{B}) \|_F$ and $c=\tilde{\mathcal{O}}(1)$. 
\end{theorem}

\begin{proof}
The derivative with respect to $\vec{W}_d$ is easy to bound:
\begin{align*}
\left\|\frac{ \partial f_{(\mathcal{W}, \mathcal{B})}(\vec{x})}{\partial \vec{W}_d} \right\| =  \|\phi(f^{(d-1)}_{(\mathcal{W}, \mathcal{B})}(\vec{x}))\| \leq \|f^{(d-1)}_{(\mathcal{W}, \mathcal{B})} (\vec{x})\|
\end{align*}

Above, we make use of the fact that for any scalar value $u$, $|\phi(u)| \leq |u|$. After applying the above inequality, $\|f^{(d-1)}_{(\mathcal{W}, \mathcal{B})}(\vec{x})\|$ can be bounded by the recursive bounds presented  in the proof of Theorem~\ref{thm:output-bound}. 

Next, for $\vec{W}_{k}$, we have that: 

\begin{align*}
\left\|\frac{ \partial f_{(\mathcal{W}, \mathcal{B})}(\vec{x})}{\partial \vec{W}_k} \right\| \leq  \left\| \vec{W}_{d} \frac{ \partial \phi(f^{(d-1)}_{(\mathcal{W}, \mathcal{B})}(\vec{x}))}{\partial \vec{W}_k}\right\| \leq \| \vec{W}_d\|_2 \left\| \frac{ \partial \phi(f^{(d-1)}_{(\mathcal{W}, \mathcal{B})}(\vec{x}))}{\partial \vec{W}_k}\right\| = \tilde{\mathcal{O}}\left((1+r)\left\| \frac{ \partial \phi(f^{(d-1)}_{(\mathcal{W}, \mathcal{B})}(\vec{x}))}{\partial \vec{W}_k}\right\|  \right)\\
\end{align*}
We have used the bound $\| \vec{W}_d\|_2 = \|\vec{Z}_d\|_2 + \| \vec{W}_d - \vec{Z}_d\|_F \leq \tilde{\mathcal{O}}(1+r)$.
Note that the last term above contains the derivative of a vector with respect to a matrix, the norm of which is essentially the norm of the gradient corresponding to every pair of term from the vector and the matrix. Now, to bound this term, we need to consider the case where $k=d-1$ and the case where $k < d-1$. However, instead of deriving the derivative for these particular cases, we will consider two more general cases, the first of which is below:
\begin{align*}
\left\| \frac{ \partial \phi( f^{(k)}_{(\mathcal{W}, \mathcal{B})}(\vec{x})) }{\partial \vec{W}_{k}}\right\|  = \left\| \frac{ \partial \phi( \vec{W}_{k} \phi(f^{(k-1)}_{(\mathcal{W}, \mathcal{B})}(\vec{x})))}{\partial \vec{W}_{k}} \right\| =  \left\|   \phi'( \vec{W}_{k} \phi(f^{(k-1)}_{(\mathcal{W}, \mathcal{B})}(\vec{x}))) \circ \phi(f^{(k-1)}_{(\mathcal{W}, \mathcal{B})}(\vec{x}) )\right\| \\
  \leq \|  \phi(f^{(k-1)}_{(\mathcal{W}, \mathcal{B})}(\vec{x}))  \leq \|f^{(k-1)}_{(\mathcal{W}, \mathcal{B})}(\vec{x}) \| 
\end{align*}

Here, $\circ$ denotes the element-wise product of two vectors. The last inequality follows from the fact that $\phi'$ is either $0$ or $1$ when $\phi$ is a ReLU activation. We can bound $\|  f^{(k-1)}_{(\mathcal{W}, \mathcal{B})}(\vec{x})\| $ with the recursive bounds presented in Theorem~\ref{thm:output-bound}.

Next, we consider the following case that remains, where $l < k$ and $k > 1$:
\begin{align*}
\left\| \frac{ \partial \phi( f^{(k)}_{(\mathcal{W}, \mathcal{B})}(\vec{x})) }{\partial \vec{W}_{l}}\right\| & = \left\| \frac{ \partial \phi( \vec{W}_{k} \phi(f^{(k-1)}_{(\mathcal{W}, \mathcal{B})}(\vec{x})))}{\partial \vec{W}_{l}} \right\| =  \left\|\phi'( \vec{W}_{k} \phi(f^{(k-1)}_{(\mathcal{W}, \mathcal{B})}(\vec{x})))  \circ  \vec{W}_{k}  \frac{ \partial  \phi(f^{(k-1)}_{(\mathcal{W}, \mathcal{B})}(\vec{x}) )}{\partial \vec{W}_{l}}\right\| \\
&\leq   \| \vec{W}_{k}\|_2\left\|   \frac{\partial  \phi(f^{(k-1)}_{(\mathcal{W}, \mathcal{B})}(\vec{x}) )}{\partial \vec{W}_{l}}\right\| \\
& \leq \tilde{\mathcal{O}} \left( (1+r)\left\| \frac{\partial \phi(f^{(k-1)}_{(\mathcal{W}, \mathcal{B})}(\vec{x}) )}{\partial \vec{W}_{l}} \right\|\right)
\end{align*}

For the sake of simplicity, we have abused notation here: in particular, in the second equality we have used $\circ$ to denote that each term in the first vector $\phi'$ is multiplied with a corresponding row in $\vec{W}_k$. Since the first vector is 0-1 vector, this results in a matrix with some rows zeroed out. The next inequality follows from the fact that the spectral norm of such a partially-zeroed-out matrix is at most the spectral norm of the original matrix.

Through these recursive bounds, we arrive at our claim.

\end{proof}

\section{Experiments}
\label{app:experiments}

All our experiments are performed on a  neural network with $4$ hidden layers (consisting of $H$ hidden units). We use SGD with a batch size of $64$; for different experiments we use different learning rates and momentum. For all the plots in the appendix, we use the CIFAR-10 and MNIST dataset with a squared error loss -- between a one-hot encoding of $k$ classes ($k=10$ here) and the output of the network --
formulated as $\mathcal{L}(\vec{f}(x), \vec{y}) = \frac{1}{k}\sum_{i=1}^{k} (f_{k}(x) - y_k)^2 $. Note that the plots in the main paper corresponding to optimizing the cross-entropy loss.

 Furthermore, while we do train the network to around $1\%$ of the initial loss for some of the experiments, for other experiments (including the one involving noisy labels) we only train the network until around $10\%$ of the original loss. While experiments in past work have studied the generalization phenomenon by training to zero or near-zero loss,  we note that it is still interesting to explore generalization without doing so because, even in this setting we still observe the unexplained phenomenon of ``non-increasing (or even, decreasing) generalization errors with increasing width''.

\subsection{Real data}
In this section, we present three sets of experiments by training A) on the MNIST data set using SGD with learning rate $0.01$ and momentum $0.9$ until the loss is less than $0.001$ , B) on the CIFAR-10 data set using SGD with learning rate $0.5$ until the loss is less than $0.02$ and C) on the MNIST data set but using using SGD with learning rate $1$. The corresponding plots can be found in Figures~\ref{fig:mnist-momentum}, ~\ref{fig:cifar} and ~\ref{fig:mnist}, where we show how the generalization error and the distance from the random initialization vary with either  the number of hidden units $H$ or the number of training samples $m$. Note that the $X$ axis is not linear but logarithmic. In some of the plots, we also use a logarithmic $Y$ axis to understand what is the power of $X$ which determines $Y$ i.e., what is $k$ if $Y \propto X^k$. Note that the generalization error only decreases or remains constant as we increase the number of hidden units.

\begin{figure}
        \begin{minipage}{.5\textwidth}
        \centering
        \includegraphics[scale=0.45]{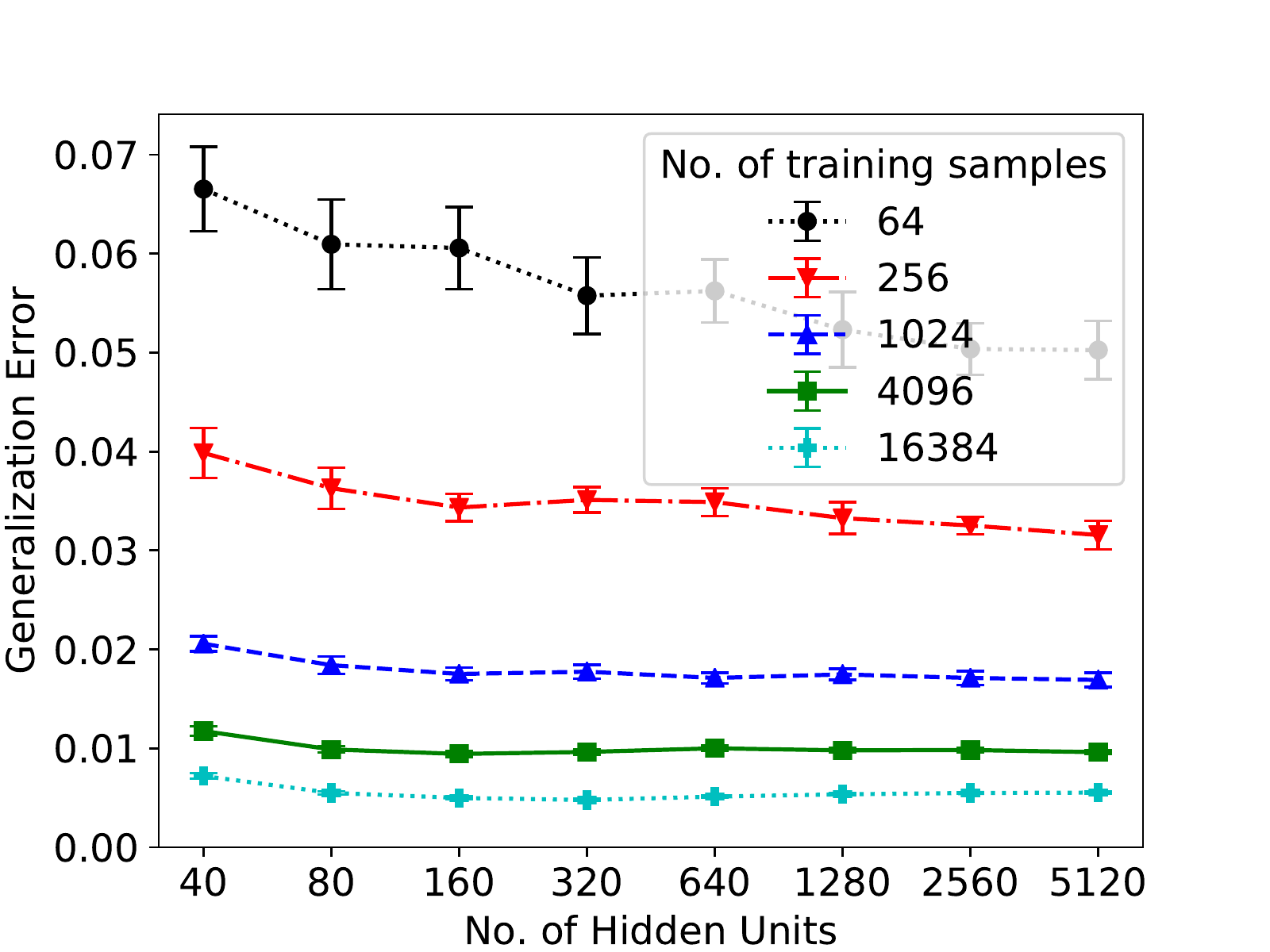}\\
        (a)
    \end{minipage}%
        \begin{minipage}{.5\textwidth}
        \centering
        \includegraphics[scale=0.45]{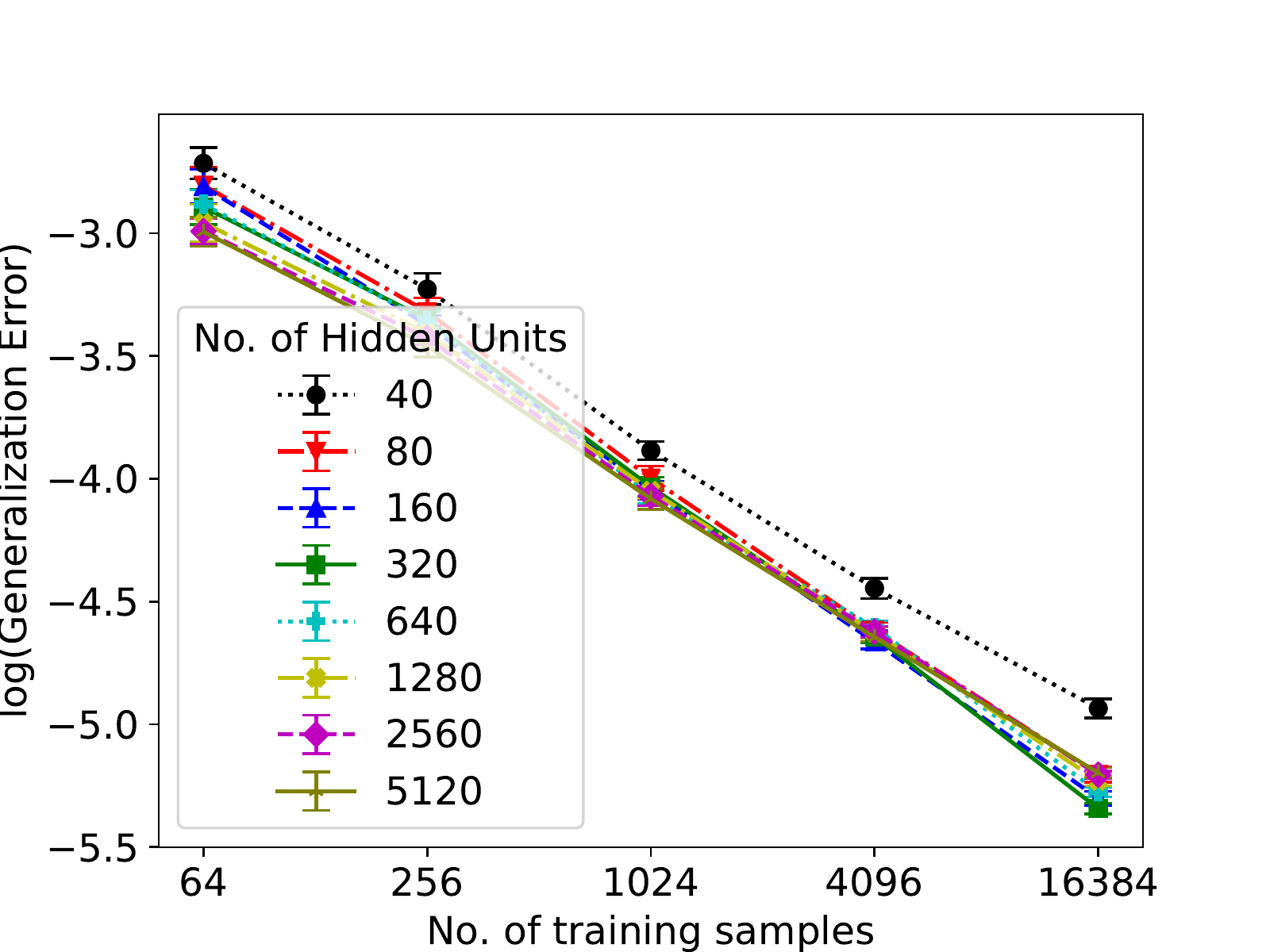}\\
        (b)
    \end{minipage}\hfill
               \begin{minipage}{.5\textwidth}
        \centering
        \includegraphics[scale=0.45]{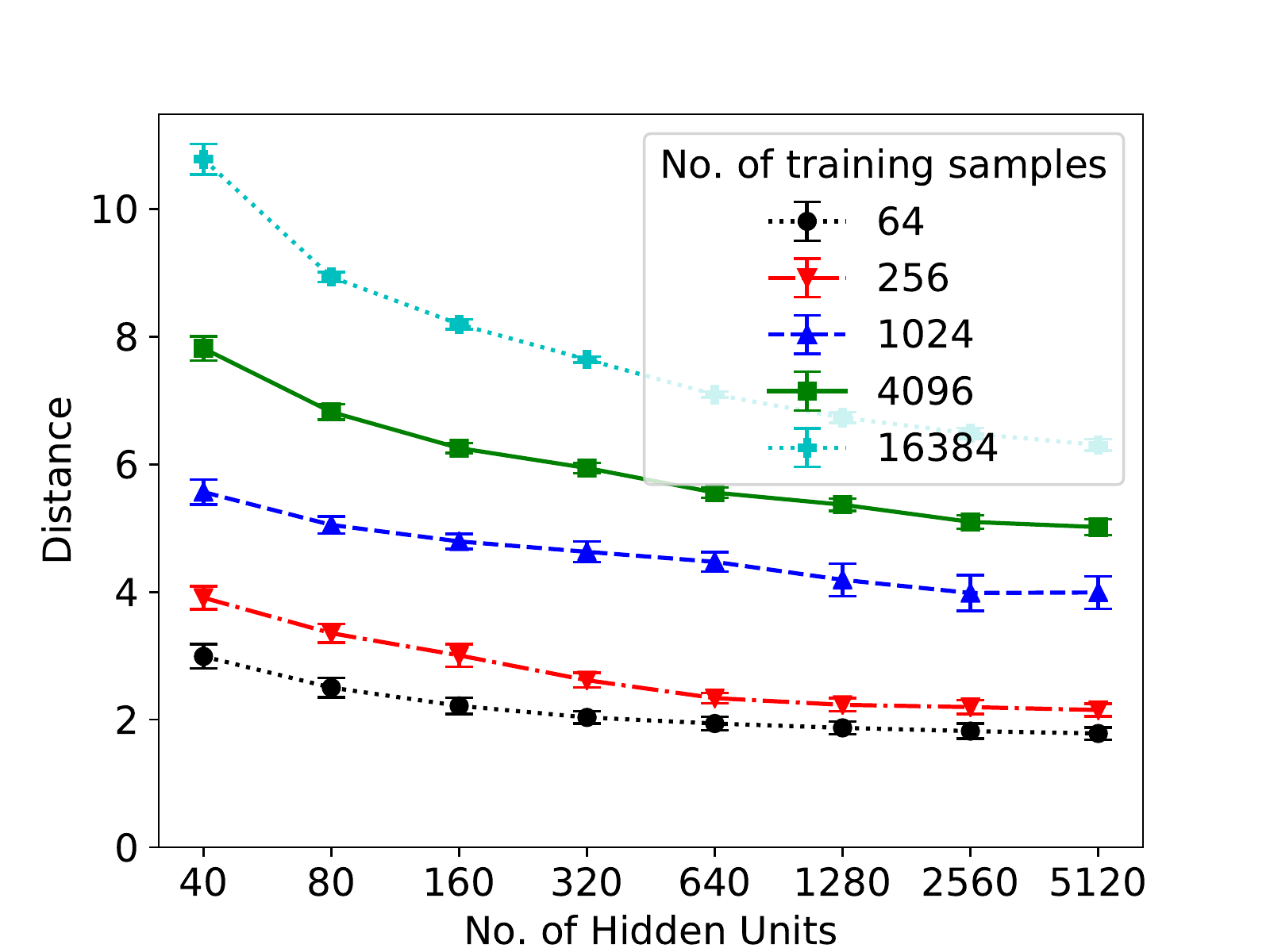} \\
		(c)	
	    \end{minipage}%
        \begin{minipage}{.5\textwidth}
        \centering
        \includegraphics[scale=0.45]{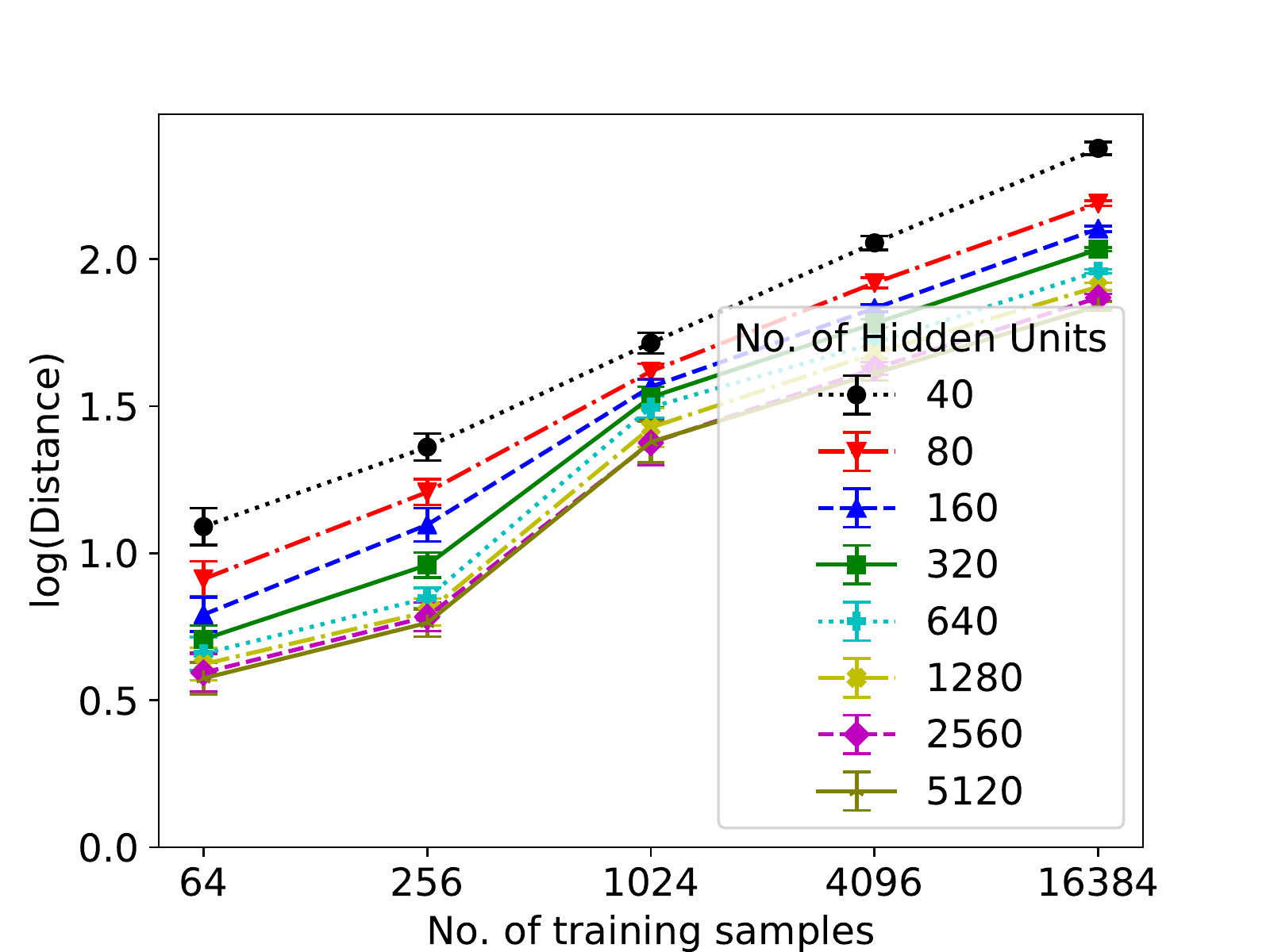}\\
	(d)
	    \end{minipage}\hfill
	\caption{Real-data: Experiment A on MNIST}
    \label{fig:mnist-momentum}
\end{figure}

\begin{figure}
        \begin{minipage}{.5\textwidth}
        \centering
        \includegraphics[scale=0.45]{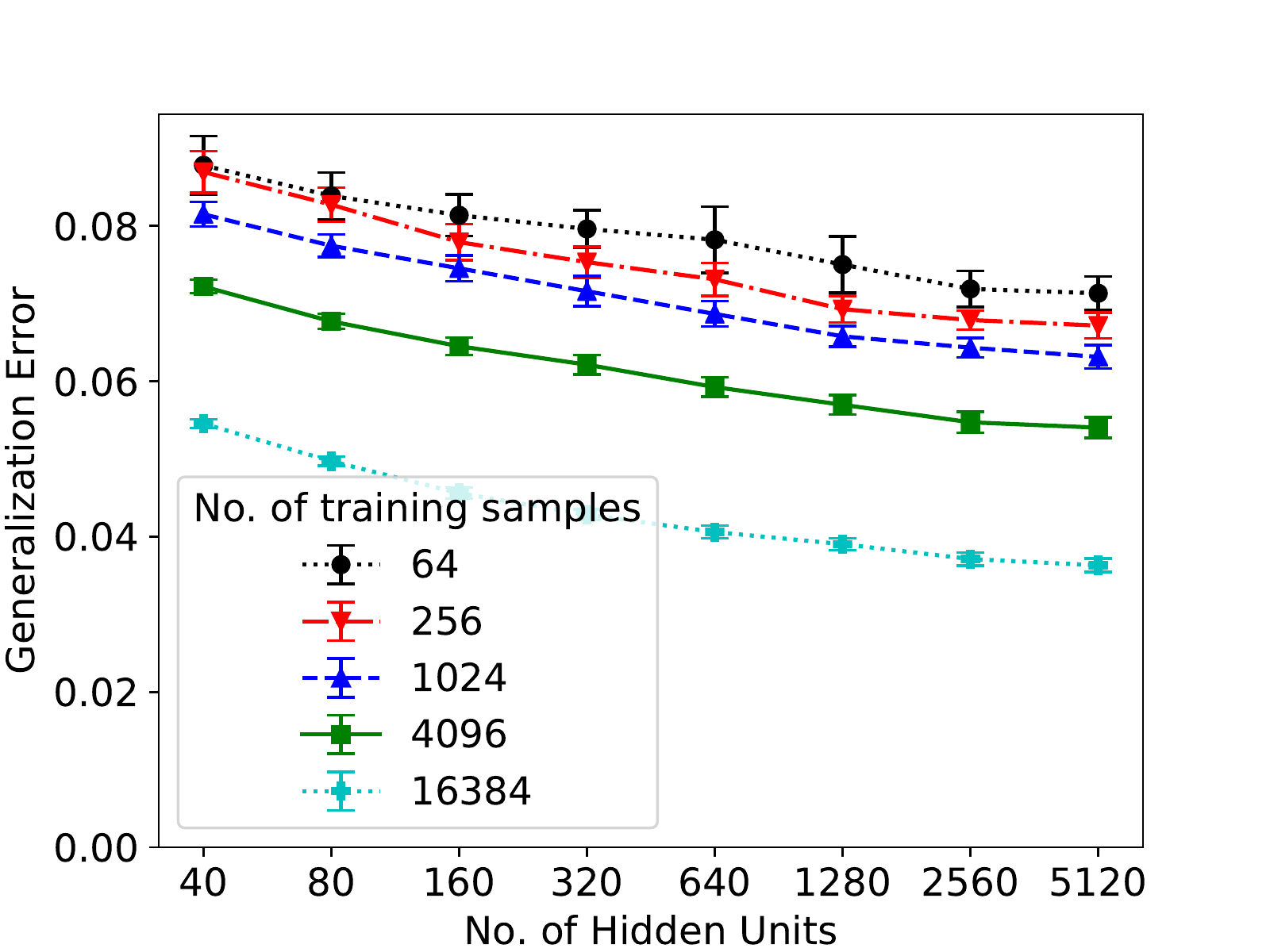}\\
        (a)
    \end{minipage}%
        \begin{minipage}{.5\textwidth}
        \centering
        \includegraphics[scale=0.45]{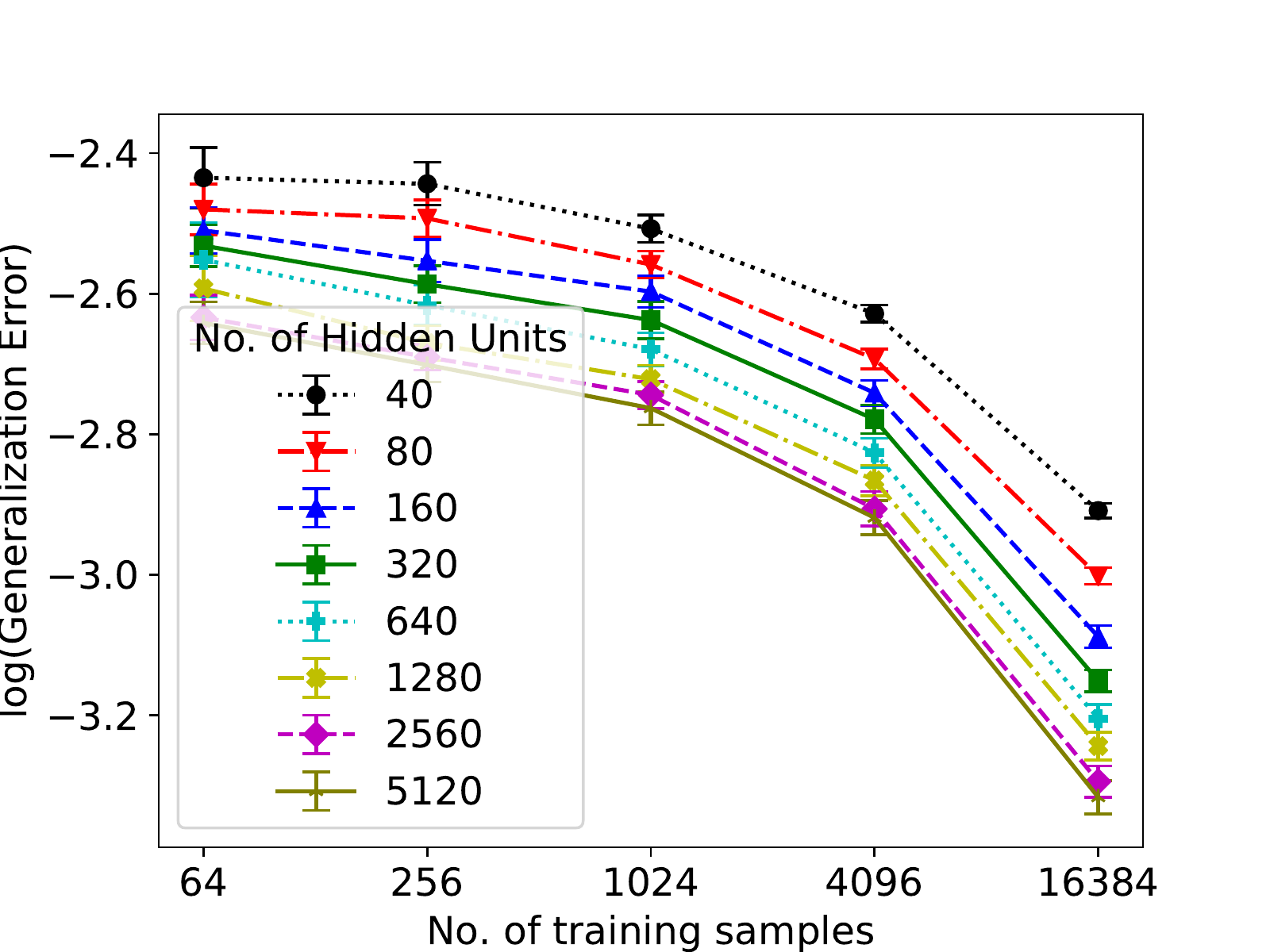}\\
        (b)
    \end{minipage}\hfill
               \begin{minipage}{.5\textwidth}
        \centering
        \includegraphics[scale=0.45]{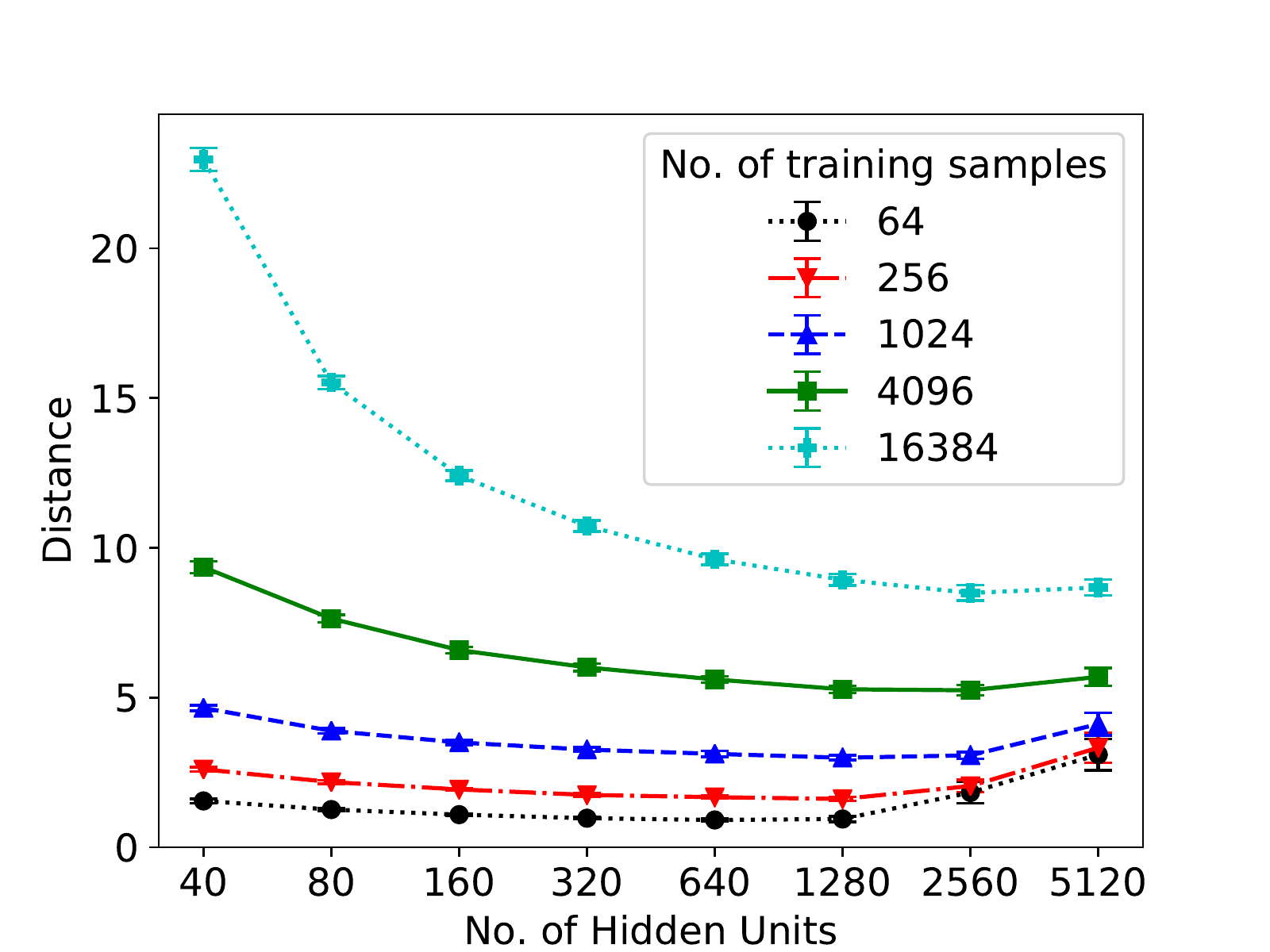} \\
		(c)	
	    \end{minipage}%
        \begin{minipage}{.5\textwidth}
        \centering
        \includegraphics[scale=0.45]{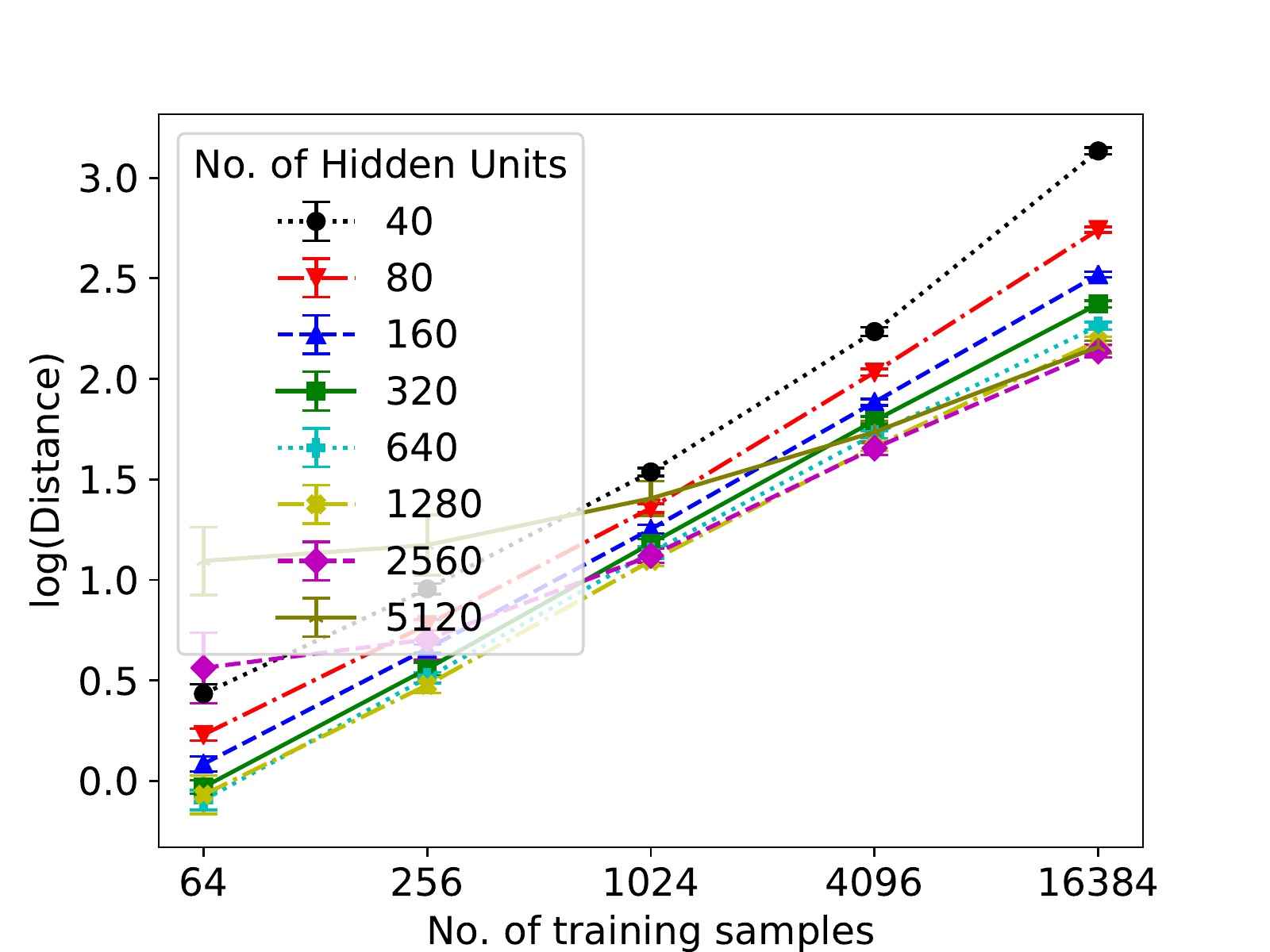}\\
	(d)
	    \end{minipage}\hfill
	\caption{Real-data: Experiment B on CIFAR-10}
    \label{fig:cifar}
\end{figure}

\begin{figure}
        \begin{minipage}{.5\textwidth}
        \centering
        \includegraphics[scale=0.45]{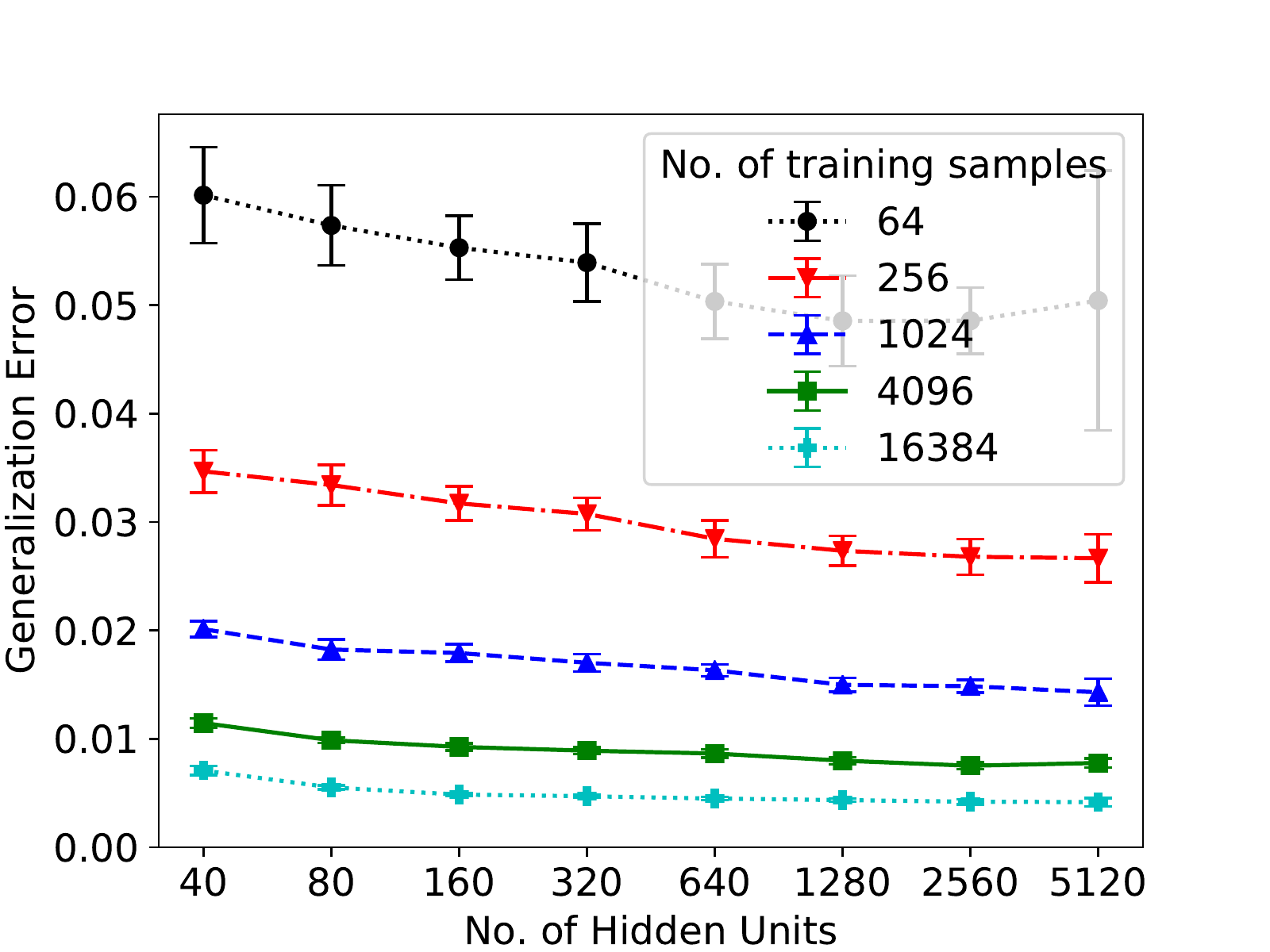}\\
        (a)
    \end{minipage}%
        \begin{minipage}{.5\textwidth}
        \centering
        \includegraphics[scale=0.45]{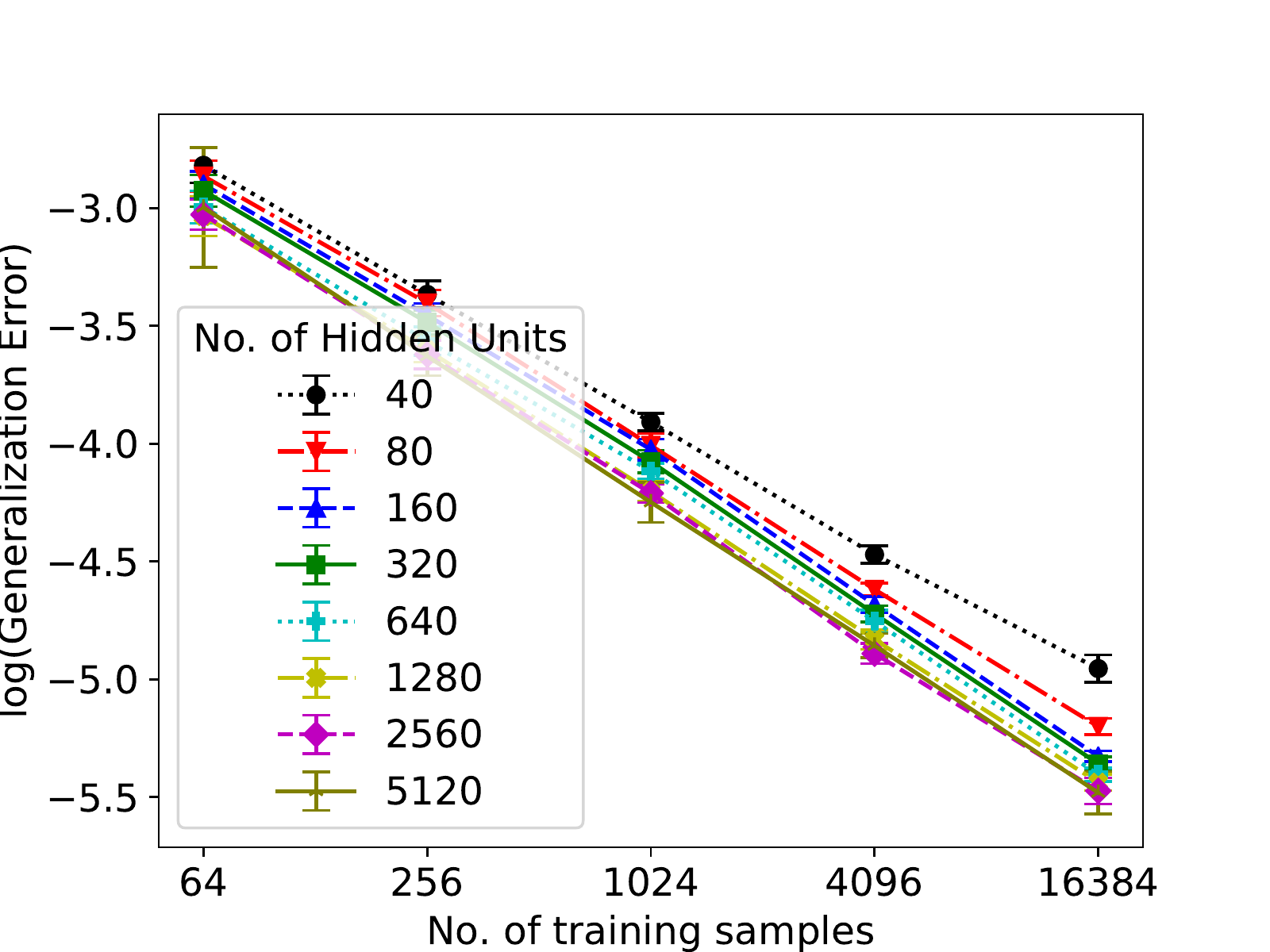}\\
        (b)
    \end{minipage}\hfill
               \begin{minipage}{.5\textwidth}
        \centering
        \includegraphics[scale=0.45]{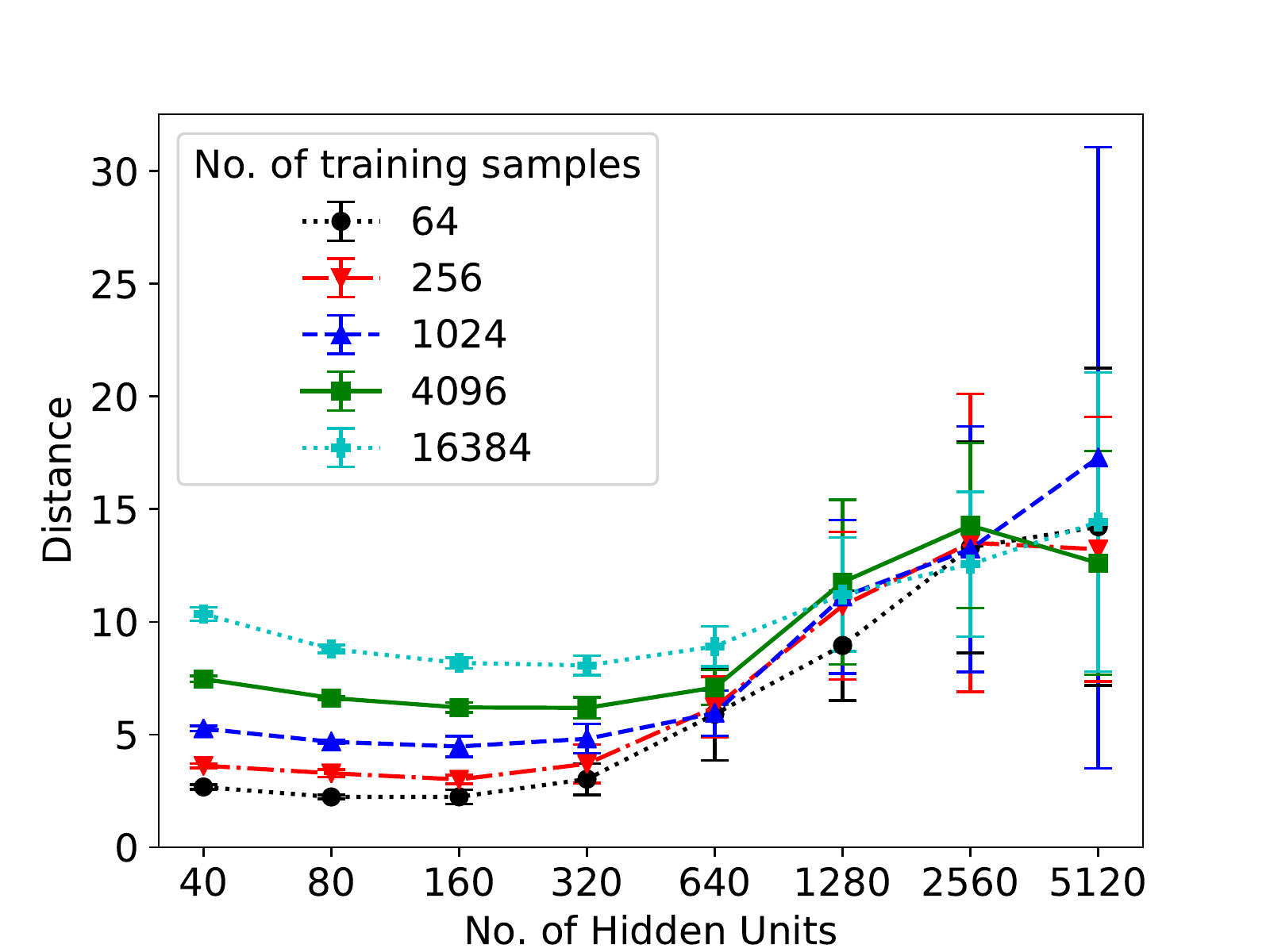} \\
		(c)	
	    \end{minipage}%
        \begin{minipage}{.5\textwidth}
        \centering
        \includegraphics[scale=0.45]{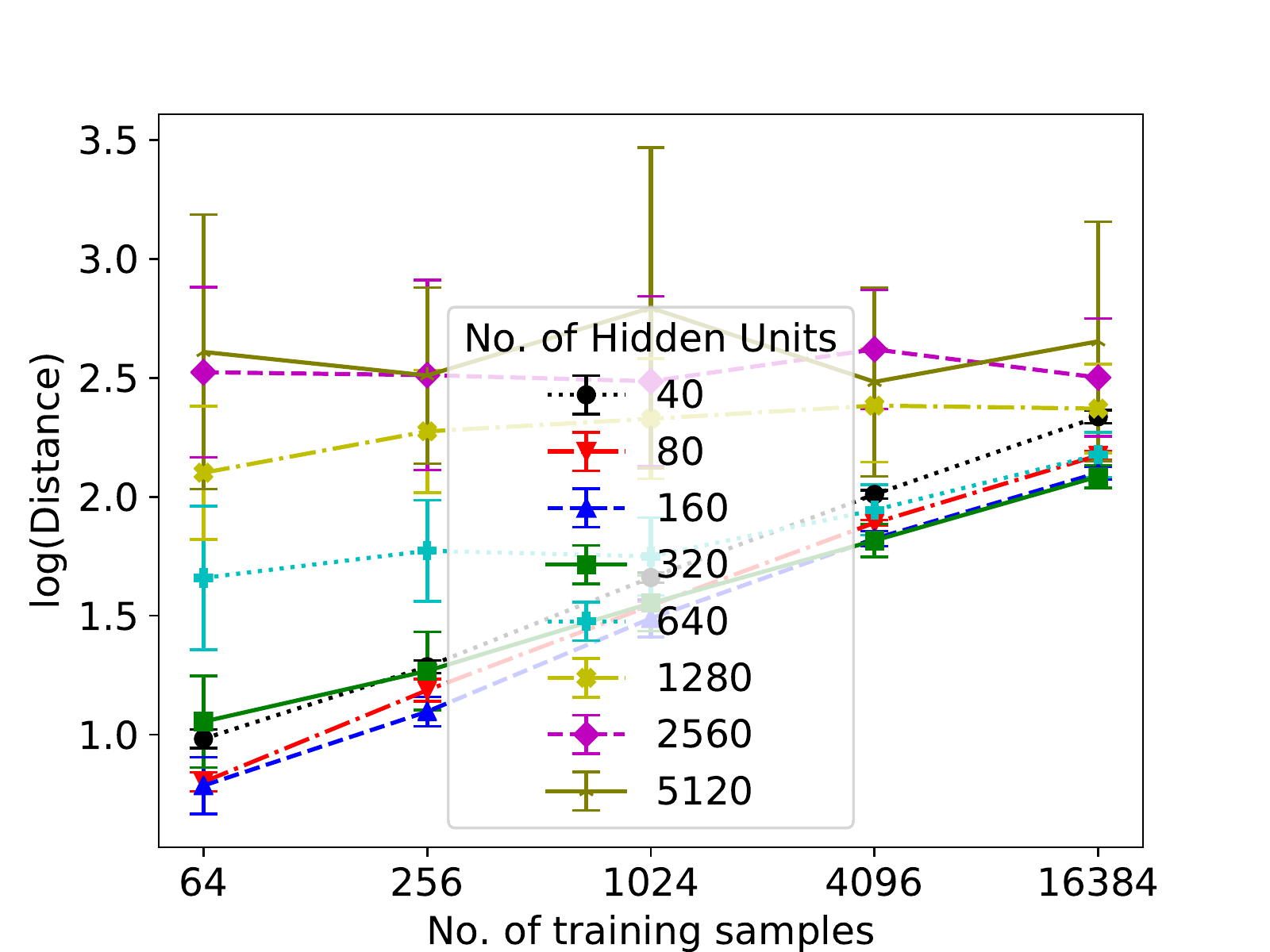}\\
	(d)
	    \end{minipage}\hfill
	\caption{Real-data: Experiment C on MNIST with larger learning rate of $1.0$, without momentum.}
    \label{fig:mnist}
\end{figure}

\subsection{Noisy Labels}
Next, we investigate how the distances vary when the labels of the datapoints are randomly chosen to be $1$ or $-1$ with equal probability independently (see Figure~\ref{fig:noise}).  For MNIST, we train using SGD with learning rate $0.01$ until the loss goes to $0.1$ and for CIFAR, we train using SGD with learning rate $0.001$ until the loss goes to $0.1$. 

\begin{figure}
        \begin{minipage}{.5\textwidth}
        \centering
        \includegraphics[scale=0.45]{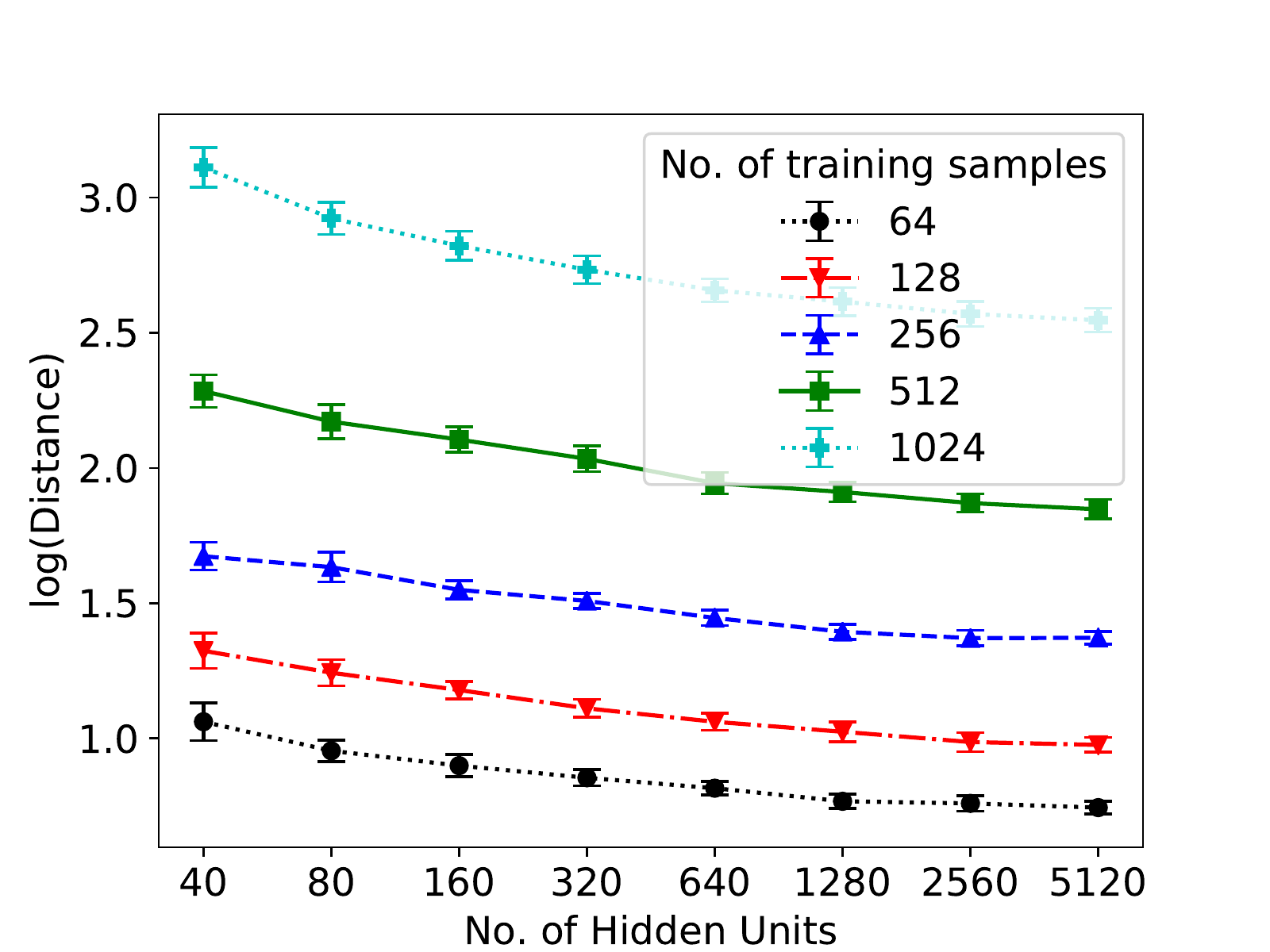}\\
        (a) MNIST
    \end{minipage}%
        \begin{minipage}{.5\textwidth}
        \centering
        \includegraphics[scale=0.45]{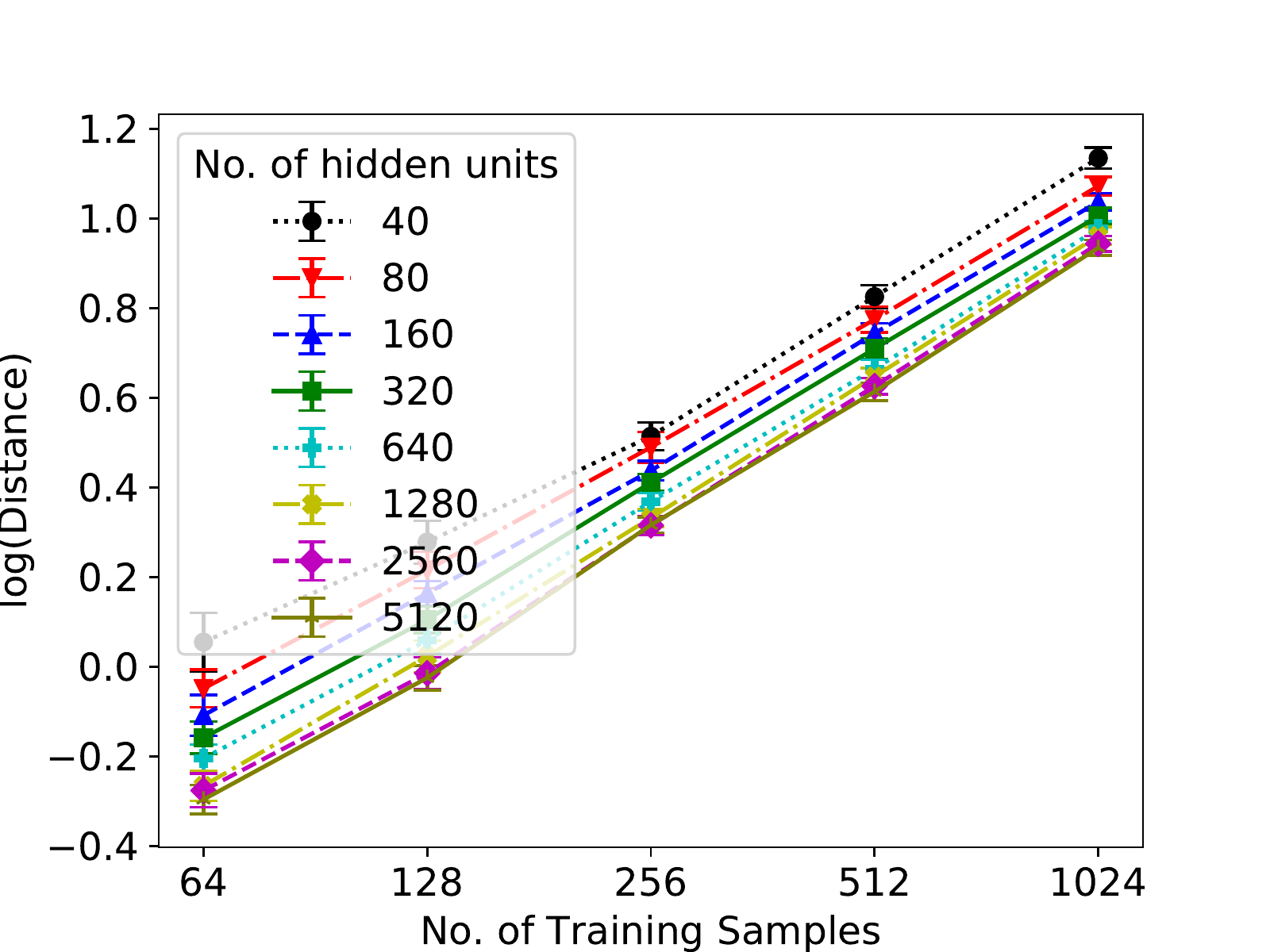}\\
        (b) MNIST
    \end{minipage}\hfill
               \begin{minipage}{.5\textwidth}
        \centering
        \includegraphics[scale=0.45]{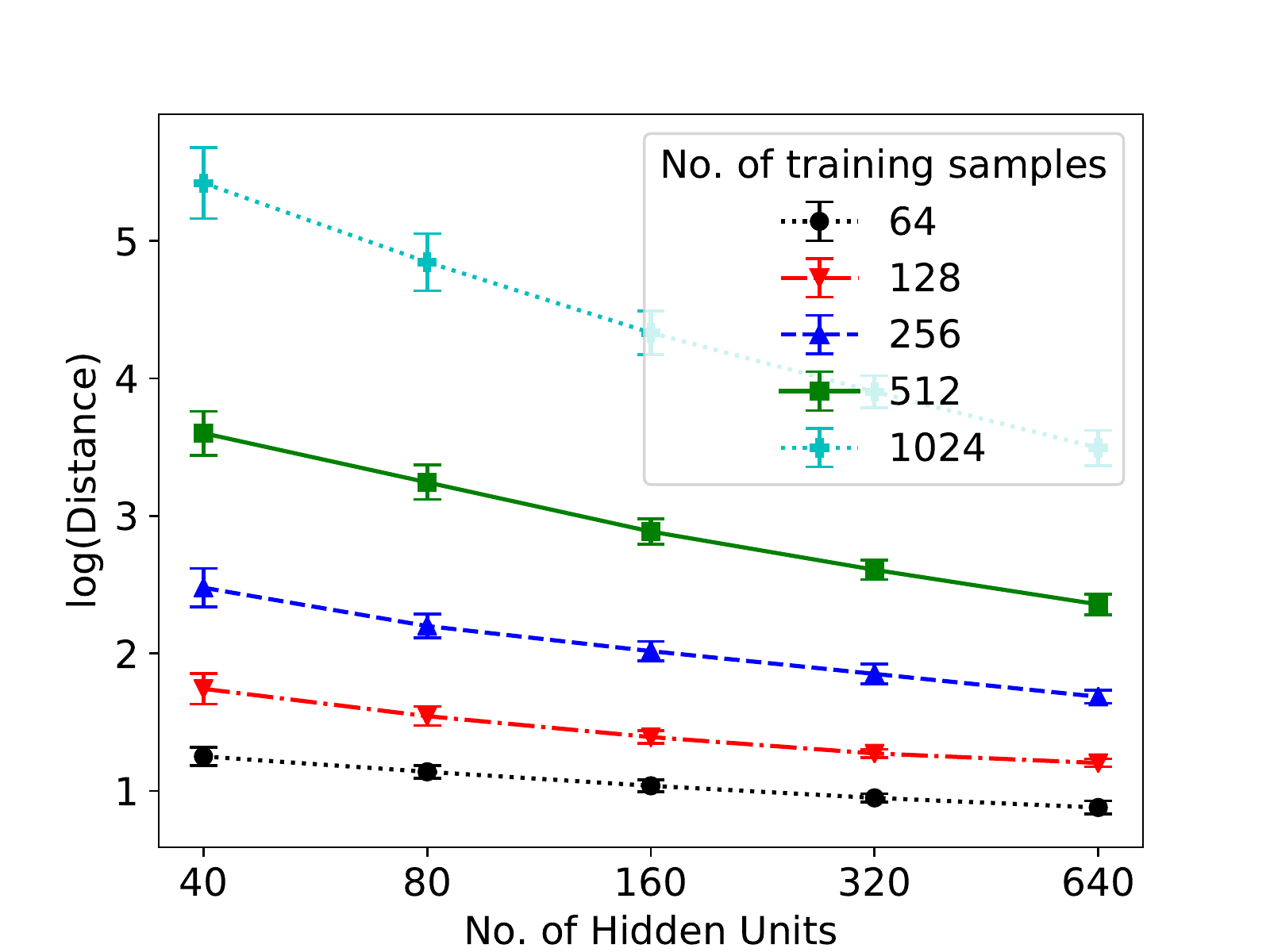} \\
		(c)	CIFAR
	    \end{minipage}%
        \begin{minipage}{.5\textwidth}
        \centering
        \includegraphics[scale=0.45]{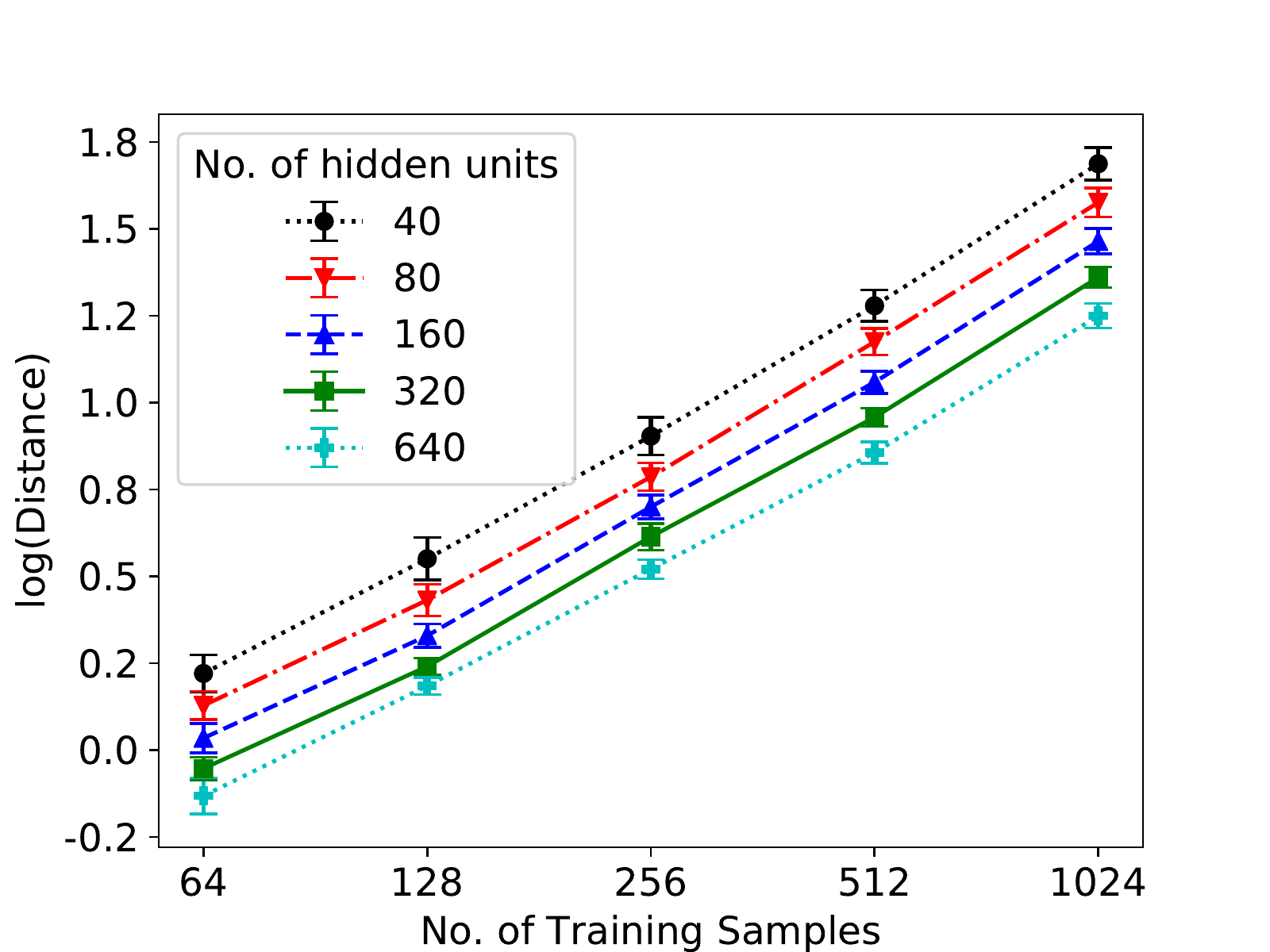}\\
	(d) CIFAR
	    \end{minipage}\hfill
	\caption{Experiments on completely random labels: On the left, we plot distance vs. $H$ for different training dataset sizes. On the right, we plot distance vs. training set size for different widths.}
    \label{fig:noise}
\end{figure}

\subsection{Partially Noisy Labels}
Next, we select datapoints of just two classes in the MNIST dataset, and corrupt a part of its labels randomly. We train using SGD with learning rate $0.01$ until the loss diminishes to $0.1$. We examine how the distance moved from the random initialization varies with the level of noise (see Figure~\ref{fig:varied-noise}).  Note that here the X axis, which is basically the proportion of points that have been corrupted. 

\begin{figure}
        \begin{minipage}{.5\textwidth}
        \centering
        \includegraphics[scale=0.45]{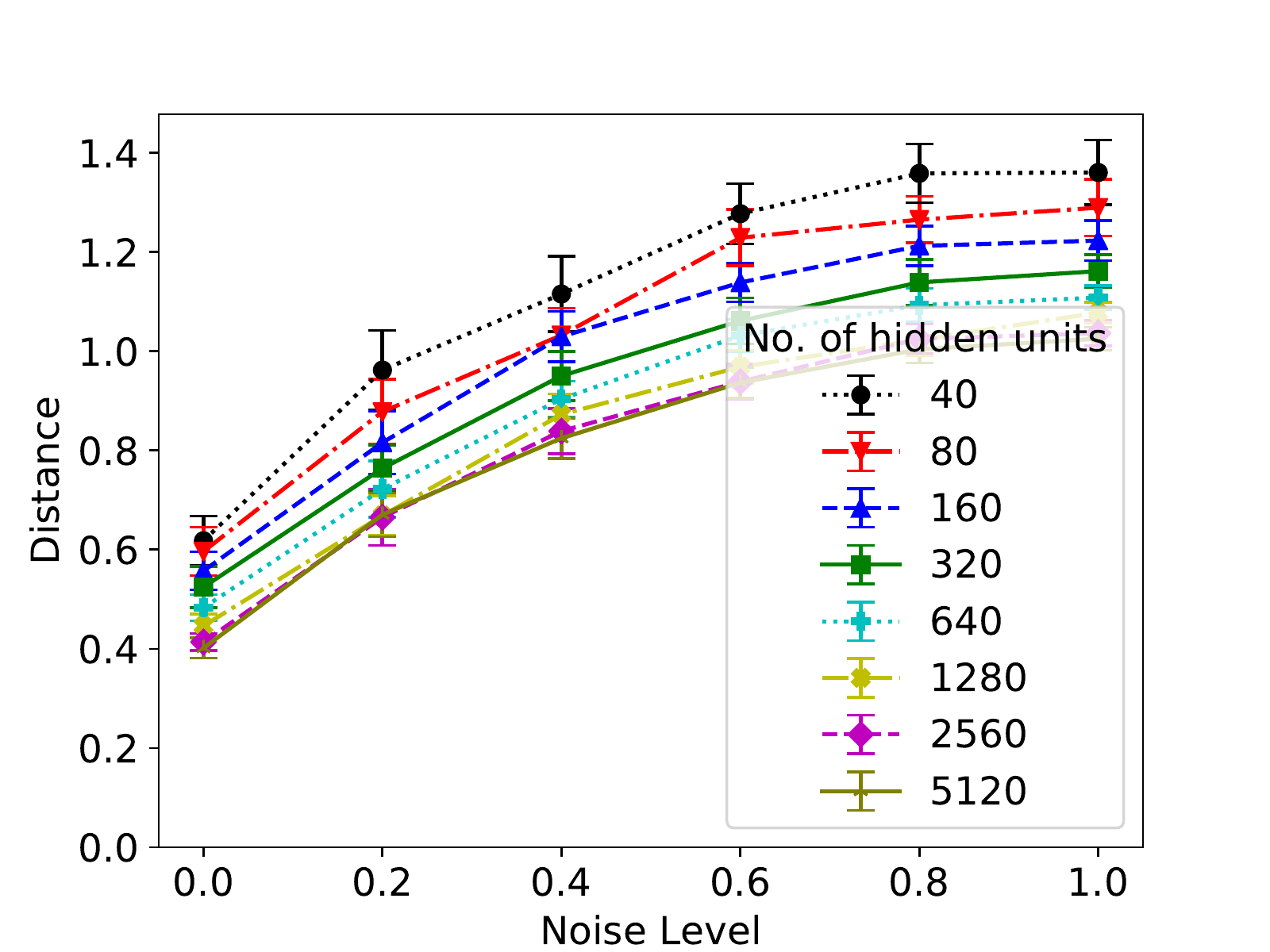}\\
        (a) $m=128$
    \end{minipage}%
        \begin{minipage}{.5\textwidth}
        \centering
        \includegraphics[scale=0.45]{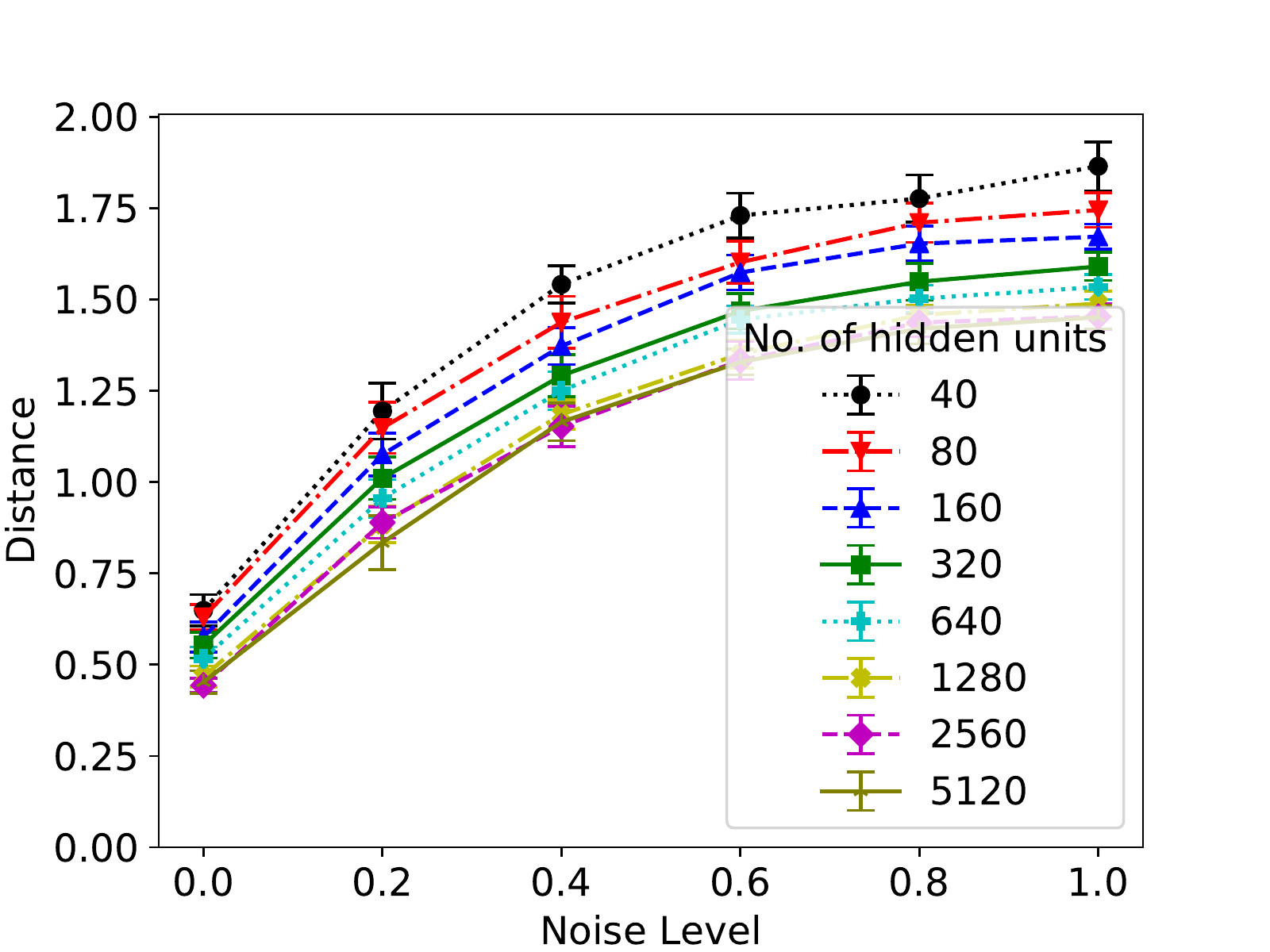}\\
        (b) $m=256$
    \end{minipage}\hfill
               \begin{minipage}{.5\textwidth}
        \centering
        \includegraphics[scale=0.45]{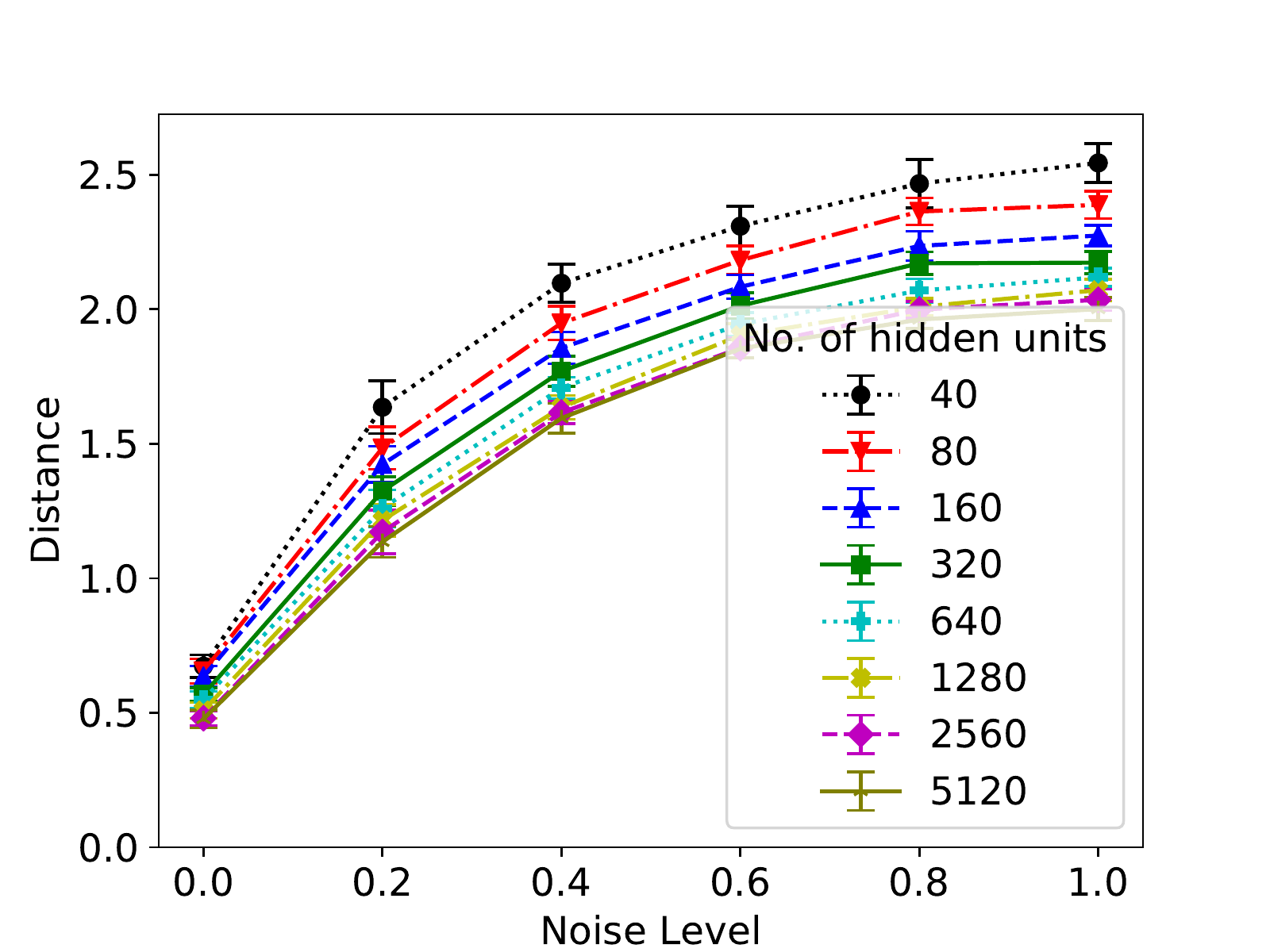} \\
		(c)	$m=512$
	    \end{minipage}%
        \begin{minipage}{.5\textwidth}
        \centering
        \includegraphics[scale=0.45]{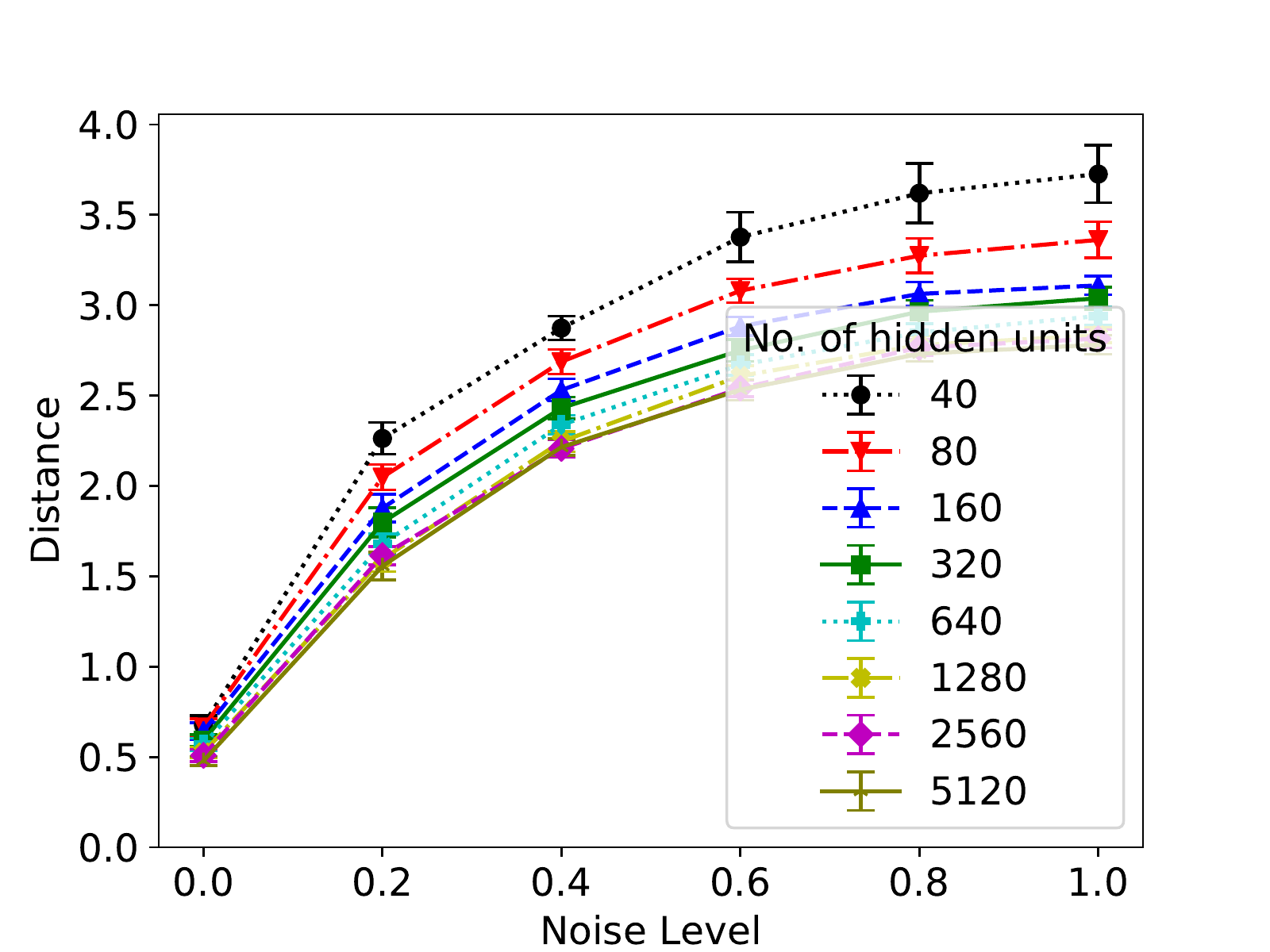}\\
	(d) $m=1024$
	    \end{minipage}\hfill
	\caption{Experiments on partially noisy labels with different levels of noise: For each plot we fix the training data size, and plot distance vs. noise for different values of $H$.}
    \label{fig:varied-noise}
\end{figure}

\end{document}